\documentclass[10pt]{article} 
\usepackage[accepted]{tmlr}

\usepackage{amsmath,amsfonts,bm}

\def\eqref#1{(\ref{#1})}

\def\ceil#1{\lceil #1 \rceil}
\def\floor#1{\lfloor #1 \rfloor}
\def\1{\bm{1}}

\def\mT{{\bm{T}}}

\DeclareMathAlphabet{\mathsfit}{\encodingdefault}{\sfdefault}{m}{sl}
\SetMathAlphabet{\mathsfit}{bold}{\encodingdefault}{\sfdefault}{bx}{n}

\DeclareMathOperator*{\argmax}{arg\,max}

\usepackage{hyperref}
\usepackage{url}

\usepackage{hyperref}
\usepackage{url}
\usepackage{mathtools}
\usepackage{amsthm}
\usepackage{amssymb}
\usepackage{amsfonts}
\usepackage[ruled,vlined,linesnumbered]{algorithm2e}
\usepackage{algorithmic}
\usepackage{booktabs}
\usepackage{caption}
\usepackage{comment}
\usepackage{subcaption}
\usepackage{float}

\usepackage{here}
\usepackage{wrapfig}
\usepackage{multirow}
\usepackage{color}
\usepackage{wrapfig}
\usepackage{textcomp}
\usepackage{comment}
\usepackage{multirow}

\newtheorem{theorem}{Theorem}
\newtheorem{lemma}{Lemma}

\newtheorem{assumption}{Assumption}

\newcommand{\argtopk}{\mathop{\mathrm{arg\,top}\text{-}k}}
\title{Necessary and Sufficient Watermark for Large Language Models}

\author{\name Yuki Takezawa \email yuki-takezawa@ml.ist.i.kyoto-u.ac.jp \\
      \addr Kyoto University, OIST
      \AND
      \name Ryoma Sato \email rsato@nii.ac.jp \\
      \addr NII
      \AND
      \name Han Bao \email bao@i.kyoto-u.ac.jp \\
      \addr Kyoto University, OIST
      \AND
      \name Kenta Niwa \email kenta.niwa.bk@hco.ntt.co.jp \\
      \addr NTT Communication Science Laboratories
      \AND
      \name Makoto Yamada \email makoto.yamada@oist.jp \\
      \addr OIST
}

\def\month{02}  
\def\year{2025} 
\def\openreview{\url{https://openreview.net/forum?id=FcyHZ6Q4k0}} 

\begin{document}

\maketitle

\begin{abstract}
Large language models (LLMs) can now generate texts that are indistinguishable from those written by humans.
Such remarkable performance of LLMs increases their risk of being used for malicious purposes. 
Thus, it is necessary to develop methods for distinguishing texts written by LLMs from those written by humans.
Watermarking is one of the most powerful methods for achieving this.
Although existing methods have successfully detected texts generated by LLMs,
they inevitably degrade the text quality.
In this study, we propose the Necessary and Sufficient Watermark (NS-Watermark) for inserting watermarks into generated texts with minimum text quality degradation.
More specifically, we derive minimum constraints required to be imposed on the generated texts to distinguish whether LLMs or humans write the texts, and we formulate the NS-Watermark as a constrained optimization problem.
Through the experiments, we demonstrate that the NS-Watermark can generate more natural texts than existing watermarking methods
and distinguish more accurately between texts written by LLMs and those written by humans.
Especially in machine translation tasks,
the NS-Watermark can outperform the existing watermarking method by up to $30$ BLEU scores.
\end{abstract}

\section{Introduction}
Large language models (LLMs) have achieved remarkable performances in a wide range of NLP tasks, 
including language generation \citep{chen2021evaluating}, question answering \citep{joshi2017triviaqa,kwiatkowski2019natural}, and reasoning tasks \citep{bisk2020piqa,kojima2022large}.
Recently, many pre-trained LLMs have been released \citep{brown2020language,chung2022scaling,zhang2022opt,touvron2023llama},
which can now generate natural and fluent texts that are indistinguishable from texts written by humans.
For instance, \citet{brown2020language} evaluated the quality of the news articles generated by GPT-3,
demonstrating that humans can hardly distinguish between news articles generated by GPT-3 and those written by humans.

As the performance of LLMs improves for various tasks, 
the risk that LLMs are used for malicious purposes, such as generating fake news, also increases \citep{zellers2019defending}.
In addition to fake news, it can be a serious concern to use bots to try to manipulate election campaigns or to use LLMs to generate academic papers \citep{kirchenbauer2023watermark}.
Thus, it is crucial to develop methods to identify whether LLMs or humans write texts.
Watermarking is one of the powerful techniques for this purpose,
which inserts information into texts such that the inserted information is imperceptible to humans and can be easily identified by some algorithms \citep{venugopal2011watermarking,he2021protecting,he2022cater,zhao2023protecting,kuditipudi2023robust,zhao2023provable,kirchenbauer2023watermark,kirchenbauer2023reliability,christ2023undetectable}.
Recently, \citet{kirchenbauer2023watermark} proposed the Hard/Soft-Watermark,
which inserts watermarks by generating text using only a subset of vocabulary.\footnote{We coined the name ``Hard/Soft-Watermark'' to refer to the watermarking methods by \citet{kirchenbauer2023watermark}.}
Texts generated by the watermarked LLMs consist only of a subset of vocabulary, whereas texts written by humans consist of an entire vocabulary.
Thus, we can identify whether LLMs or humans write texts using statistical hypothesis testing.
\citet{kirchenbauer2023watermark} demonstrated that texts generated by LLMs with the Hard/Soft-Watermark can be distinguished from human-written texts almost perfectly.
However, the generated texts are often low-quality because LLMs generate texts with only a subset of vocabulary.

In this study, we propose a novel method for inserting watermarks into generated text without sacrificing both text quality and detection accuracy,
which we refer to as the \textbf{Necessary and Sufficient Watermark (NS-Watermark)}.
Our method is based on the observation that the constraint imposed by the Hard/Soft-Watermark is overly conservative for identifying LLM-generated texts, especially when the generated texts are long.
Hence, we derive minimum constraints required to be imposed on the generated texts to detect LLM-generated texts.
We find that the constraints on the generated text can be relaxed without decreasing the detection accuracy as the length of the generated text increases.
Based on this observation, we propose the NS-Watermark, 
which can change the constraints according to the length and impose minimum constraints on the generated text.
Owing to the minimum constraints, the text generated with the NS-Watermark can be more natural than the text with the Hard/Soft-Watermark.
We experimentally evaluate the effectiveness of the NS-Watermark
and demonstrate that the NS-Watermark can outperform the Soft-Watermark in terms of both text quality and detection accuracy.
Particularly in the machine translation tasks,
we demonstrate that the NS-Watermark can outperform the Soft-Watermark by up to $30$ BLEU scores
and achieve competitive BLEU scores compared to conventional decoding methods without watermarks.

\section{Background}
In this section, we briefly describe the watermarking methods proposed by \citet{kirchenbauer2023watermark}.
Further discussions on related studies are deferred to Sec.~\ref{sec:related_work}.

\textbf{Hard-Watermark.}
Let $x_{\text{prompt}}$ be a prompt, $V$ be vocabulary, and $\gamma \in (0,1)$ be a hyperparameter.
Given a word $x_t$, 
using $x_t$ as the seed value,
we randomly split $V$ into two disjoint subsets: \emph{green words} $V^{\text{green}} (x_t)$ and \emph{red words} $V^{\text{red}} (x_t) (\coloneqq V \setminus V^{\text{green}} (x_t))$
such that $|V^{\text{green}} (x_t)| = \gamma |V|$. 
Then, the Hard-Watermark generates text as follows:
\begin{align}
    \label{eq:hard_watermark}
    \argmax_{x_{1:T}, T} p(x_{1:T} \mid x_{\text{prompt}}) 
    \;\; \text{s.t.} \;\; x_{t+1} \in V^{\text{green}}(x_t) \;\; (t = 1, 2, \cdots, T-1).
\end{align}
We randomly split vocabulary into red and green words, and humans do not know whether a word is green or red.
Thus, if humans write the text, the green words appear with probability $\gamma$,
whereas texts written by LLMs consist of only green words.
Thus, we can use statistical hypothesis testing to identify whether LLMs or humans write the text. Specifically, the null and alternative hypotheses are given as follows:
\begin{itemize}
    \item[$H_0$:] The green words appear in a text with probability $\gamma$.
    \item[$H_1$:] The green words appear in a text with a probability greater than $\gamma$.
\end{itemize}
When the null hypothesis is rejected, we conclude that the text is generated by LLMs.
The number of green words follows a binomial distribution in texts written by humans.
Thus, we can test this by checking whether the z-score of text $x_{1:T}$, defined below, exceeds a given threshold $Z$.
\begin{equation}
    z(x_{1:T}) \coloneqq \frac{|x_{1:T}|_{\mathrm{G}} - \gamma (T-1)}{\sqrt{\gamma (1-\gamma) (T-1)}},
\end{equation}
where $|x_{1:T}|_\mathrm{G} \coloneqq \left| \{ x_{t+1} \mid x_{t+1} \in V^{\text{green}} (x_t) \} \right|$.

\textbf{Soft-Watermark.}
Although the Hard-Watermark is a simple and efficient method for distinguishing LLM-generated texts from those written by humans,
the generated texts are often of low quality. 
This is partly because the constraints of the Hard-Watermark may prevent the generation of common phrases, e.g., ``Barack Obama,''
even if the probabilities of these phrases are very high 
just because ``Obama'' is not contained in $V^{\text{green}}(\text{``Barack''})$.
To mitigate this issue, \citet{kirchenbauer2023watermark} also proposed the Soft-Watermark:
instead of making all words contained in the generated text green words,
the Soft-Watermark adds an offset and increases the probability of generating green words.
This relaxation allows the Soft-Watermark to generate ``Barack Obama'' when the probability that ``Obama'' appears after ``Barack'' is high,
and the Soft-Watermark can generate higher-quality text than the Hard-Watermark.
However, the Soft-Watermark still suffers from low-quality text,
as we demonstrate in the experiments.

\section{Proposed Method}
\begin{figure*}[b!]
    \centering
    \begin{subfigure}{0.64\hsize}
        \centering
        \includegraphics[width=0.49\hsize]{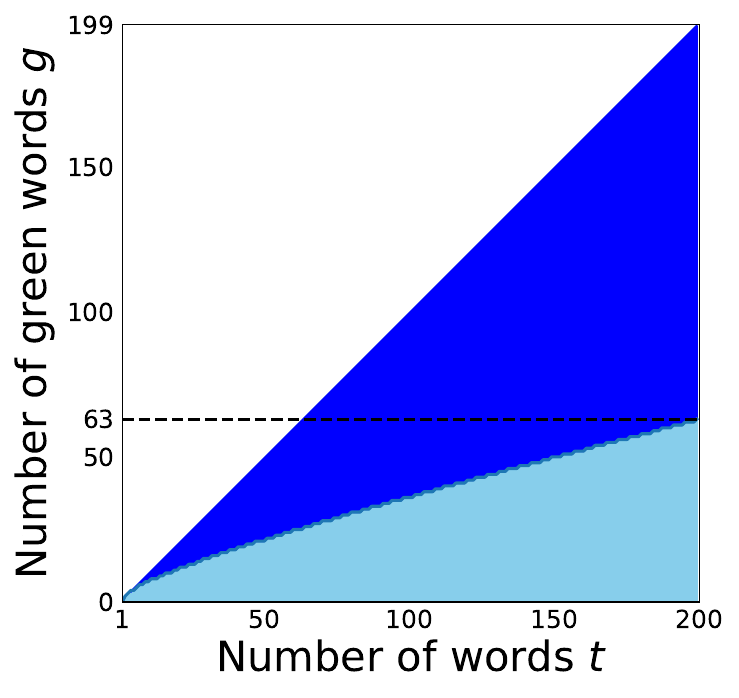}
        \includegraphics[width=0.49\hsize]{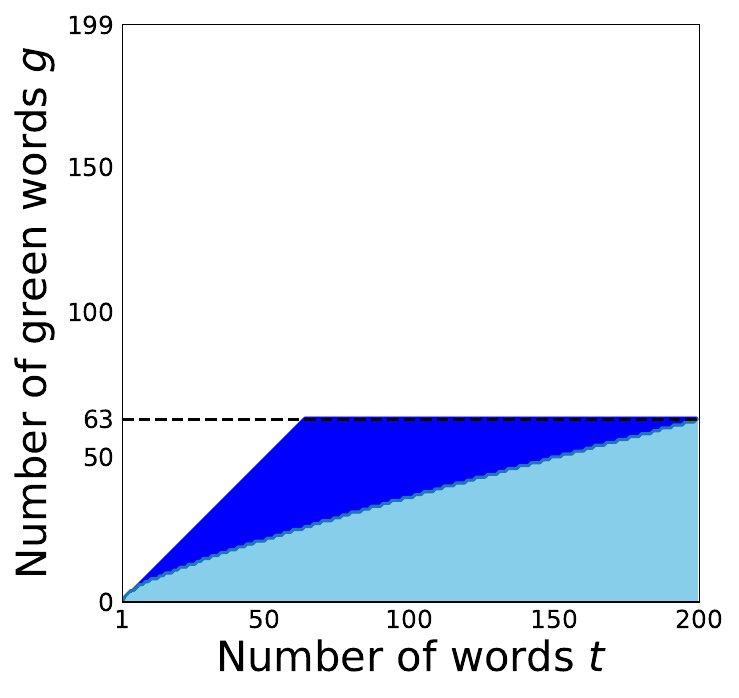}
        \caption{Illustration of Alg.~\ref{alg:ns_watermark}}
        \label{fig:alg1}
    \end{subfigure}
    \begin{subfigure}{0.32\hsize}
        \centering
        \includegraphics[width=\hsize]{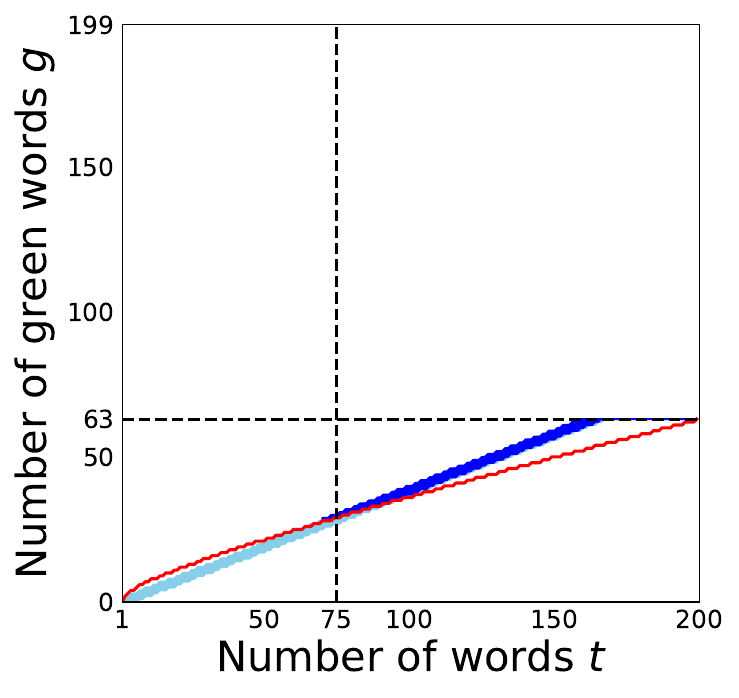}
        \caption{Illustration of Alg.~\ref{alg:fast_ns_watermark}}
        \label{fig:alg2}
    \end{subfigure}
    \vskip - 0.05 in
    \caption{Visualization of the table $\mT[t][g]$ for $T_{\text{max}}=200$, $\gamma=0.2$, $\widehat{T}=75$, $\alpha=2$, $Z=4$, and $G_{\text{max}}=63$. The areas colored in blue and light blue indicate the range in $\mT[t][g]$ where we need to calculate, and the areas colored in blue indicate the range that satisfies the constraint of Eq.~\eqref{eq:ns_watermark}. The red line indicates the minimum number of green words required to satisfy the constraint. Note that in the middle and right figures, $\mT[t][G_{\text{max}}]$ does not denote the set of texts of length $t$ containing $G_{\text{max}}$ green words, but denotes the set of texts containing at least $G_{\text{max}}$ green words. See Sec.~\ref{sec:additional_visual_explation} for figures with various $\gamma$.}
    \label{fig:dp}
\end{figure*}

\subsection{Necessary and Sufficient Watermark}
\label{sec:problem_formulation}
In this section, we show that the constraints of the Hard/Soft-Watermark are too restrictive 
and derive the minimum constraint to identify whether LLMs or humans write the text.
By rewriting Eq.~\eqref{eq:hard_watermark},
the Hard-Watermark is reformulated as follows:
\begin{equation}
\label{eq:hard_watermark2}
    \argmax_{x_{1:T}, T} p(x_{1:T} \mid x_{\text{prompt}}) \;\; \text{s.t.} \;\; \frac{|x_{1:T}|_\mathrm{G}}{T-1} = 1.
\end{equation}
Let $\hat{x}_{1:T}$ be the solution of Eq.~\eqref{eq:hard_watermark2}.
The z-score $z (\hat{x}_{1:T})$ is $\mathcal{O}(\sqrt{T})$, and we can identify whether LLMs or humans write the text by testing whether the z-score exceeds the hyperparameter $Z$.
However, the z-score $z (\hat{x}_{1:T})$ increases with the length of the generated text $T$,
whereas the threshold $Z$ remains constant.
Therefore, the above formulation in Eq.~\eqref{eq:hard_watermark2} imposes too restrictive constraint on the generated text,
especially when ensuring $z(x_{1:T}) \geq Z$ for long texts.

Alternatively, the following constraint is sufficient to ensure that the z-score of the generated text is greater than or equal to the threshold $Z$:
\begin{align}
\label{eq:ns_watermark}
    \argmax_{x_{1:T}, T} p(x_{1:T} \mid x_{\text{prompt}}) 
    \;\; \text{s.t.} \;\; \frac{|x_{1:T}|_\mathrm{G}}{T-1} \geq \gamma + Z \sqrt{\frac{\gamma (1-\gamma)}{T-1}}.
\end{align}
If text $x_{1:T}$ is written by humans, the proportion of green words contained in a text $\frac{|x_{1:T}|_G}{T-1}$ is $\gamma$ on average.
Thus, the second term in the constraint is the minimum margin for identifying whether the texts are written by LLMs.
We refer to the above problem as the \textbf{Necessary and Sufficient Watermark (NS-Watermark)}.
By comparing Eq.~\eqref{eq:hard_watermark2} with Eq.~\eqref{eq:ns_watermark}, 
the constraint of the NS-Watermark is looser than that in Eq.~\eqref{eq:hard_watermark2},
although the z-score of the generated text is guaranteed to be greater than or equal to $Z$ because of the constraint in Eq.~\eqref{eq:ns_watermark}.
Thus, the NS-Watermark can generate higher quality and more natural texts than the Hard/Soft-Watermark without decreasing detection accuracy.
In the next section, we propose an efficient algorithm for computing the NS-Watermark.

\subsection{Naive Algorithm for Necessary and Sufficient Watermark}
\label{sec:ns_watermark}

The Hard/Soft-Watermark can be computed using the conventional beam search
because the Hard-Watermark generates texts using only green words, and the Soft-Watermark just adds an offset to the probability that green words appear.
However, the NS-Watermark needs to control the proportion of green words contained in generated texts
and needs to optimize where green words should be inserted.
Moreover, the constraint in Eq.~\eqref{eq:ns_watermark} depends on the length of the generated text $T$, 
which is unknown until the text is generated.
This makes solving the NS-Watermark more challenging, which hinders the application of the conventional beam search to the NS-Watermark.
In this section, we propose an algorithm to solve the NS-Watermark. 

Let $k$ be the beam size.
Let $\mT[t][g]$ be a set of $k$ texts of length $t$ containing $g$ green words.
For simplicity, we explain the cases in which $1 \leq g$ and $1 \leq t$.
Texts of length $t+1$ containing $g$ green words can be generated by adding a green word to texts of length $t$ containing $g-1$ green words
or adding a red word to texts of length $t$ containing $g$ green words.
Formally, we generate text of length $t+1$ containing $g$ green words as follows:

\noindent
\begin{align*}
     &X_1 \! = \! \{ x_{1:t+1} \! \mid \! x_{1:t} \in \mT[t][g \! - \! 1], \; x_{t+1} \!\in\! V^{\text{green}} (x_{t}) \}, \\
     &X_2 \! = \! \{ x_{1:t+1} \! \mid \! x_{1:t} \in \mT[t][g], \; x_{t+1} \!\in\! V^{\text{red}} (x_{t}) \}, \\
    &\mT[t+1][g] \! = \! \argtopk_{x_{1:t+1} \in X_1 \cup X_2} p (x_{1:t+1} \mid x_{\text{prompt}}).
\end{align*}
By calculating $\mT[t][g]$ for all $g$ and $t$ and generating the text with the highest probability among the texts that satisfy the constraint in Eq.~\eqref{eq:ns_watermark},
we can solve Eq.~\eqref{eq:ns_watermark}.

Let $T_{\text{max}}$ a hyperparameter that controls the maximum length of generated texts.
We need to fill the table $\mT[t][g]$ for all $(t,g) \in \{ (t,g) \in \mathbb{Z}^2 \mid 1 \leq t \leq T_{\text{max}}, 0 \leq g \leq t-1  \}$,
which requires the time complexity $\mathcal{O}(k T_{\text{max}}^2)$ (see the left figure in Fig.~\ref{fig:alg1}).
However, if $G_{\text{max}} (\coloneqq \ceil{\gamma (T_{\text{max}}-1) + Z \sqrt{\gamma (1 - \gamma) (T_{\text{max}}-1)}})$ green words appear after generating $t$ words, 
it is not necessary to count the number of green words that appear in the remaining text
because the constraint in Eq.~\eqref{eq:ns_watermark} is satisfied regardless of the remaining text.
Based on this observation, we can reduce the time complexity by changing $\mT[t][G_{\text{max}}]$ to store texts of length $t$ containing \emph{at least} $G_{\text{max}}$ green words, instead of texts containing exactly $G_{\text{max}}$ green words.
Owing to this modification, we do not need to count the number of green words after $G_{\text{max}}$ green words appear in texts
and can reduce the time complexity to $\mathcal{O}(\gamma k T_{\text{max}}^2)$.
We provide a visual explanation in the figure on the right side of Fig.~\ref{fig:alg1} and show the pseudo-code in Alg.~\ref{alg:ns_watermark}.

\begin{algorithm*}[t]
\DontPrintSemicolon 
\KwIn{Maximum number of words $T_{\text{max}}$, vocabulary $V$, beam size $k$, and hyperparameter $\gamma$, $Z$.}
$G_{\text{max}} \leftarrow \ceil{\gamma (T_{\text{max}}-1) + Z \sqrt{\gamma (1 - \gamma) (T_{\text{max}}-1)}}$. \\
Let $\mT$ be a $T_{\text{max}} \times (G_{\text{max}}+1)$ table and $S$ be an empty set. \\
\SetKwProg{Fn}{Function}{ is}{end}
\Fn{update($X$: feasible set, $t$: the number of words $g$: the number of green words)}{
    $\mT[t][g] \leftarrow \emptyset$. \\
    \While{$|\mT[t][g] | < k$}{
        $x_{1:t} = \argmax_{x_{1:t} \in X \setminus \mT[t][g]} p(x_{1:t} \mid x_{\text{prompt}})$. \\
        \uIf{the last word $x_t$ is \textsc{eos}}{
            \uIf{$g \geq \gamma (t-1) + Z \sqrt{\gamma (1-\gamma) (t-1)}$}{
            $S \leftarrow S \cup \{x_{1:t}\}$.
        }
        }\Else{
            $\mT[t][g] \leftarrow \mT[t][g] \cup \{x_{1:t}\}$.
        }
    }
}
\Fn{feasible\_set($t$: the number of words, $g$: the number of green words)}{
    \uIf{$g = 0$}{
        $X \leftarrow \{ x_{1:t} \mid x_{1:t-1} \in \mT[t-1][g], x_t \in V^{\text{red}} (x_{t-1}) \}$.
    }\uElseIf{$g = t-1$}{
        $X \leftarrow \{ x_{1:t} \mid x_{1:t-1} \in \mT[t-1][g-1], x_t \in V^{\text{green}} (x_{t-1}) \}$.
    }\uElseIf{$ 1 \leq g < G_{\text{max}}$}{
        $X \leftarrow \{ x_{1:t} \mid x_{1:t-1} \in \mT[t-1][g-1], x_t \in V^{\text{green}} (x_{t-1}) \}$. \\
        $X \leftarrow X \cup \{ x_{1:t} \mid x_{1:t-1} \in \mT[t-1][g], x_t \in V^{\text{red}} (x_{t-1}) \}$.
    }\uElseIf{$g = G_{\text{max}}$}{
        $X \leftarrow \{ x_{1:t} \mid x_{1:t-1} \in \mT[t-1][g-1], x_t \in V^{\text{green}} (x_{t-1}) \}$. \\
        $X \leftarrow X \cup \{ x_{1:t} \mid x_{1:t-1} \in \mT[t-1][g], x_t \in V \}$.
    }
    \Return{$X$}
}
$\mT[1][0] \leftarrow \argtopk_{x_1 \in V} p(x_1 \mid x_{\text{prompt}})$. \\
\For{$t = 2, \cdots, T_{\text{max}}$}{
    \For{$g = 0, \cdots, \min \{ t-1, G_{\text{max}}\}$}{
        $X \leftarrow \textit{feasible\_set}(t, g)$. \\
        $\textit{update}(X, t, g)$.
    }
}
\Return{$\argmax_{x_{1:t} \in S \cup \mT[T_{\text{max}}][G_{\text{max}}]} p(x_{1:t} \mid x_{\text{prompt}})$.}
\caption{Naive algorithm for the NS-Watermark.}
\label{alg:ns_watermark}
\end{algorithm*}

\subsection{Linear Time Algorithm for Necessary and Sufficient Watermark}
\label{sec:linear}

Algorithm~\ref{alg:ns_watermark} in the previous section has the time complexity of $\mathcal{O}(\gamma k T_{\text{max}}^2)$, which is practically not suitable for LLMs with an extremely large number of parameters.
We resort to an approximation and reduce the time complexity to linear.


The major bottleneck of the quadratic time complexity in $T_{\text{max}}$ is that Alg.~\ref{alg:ns_watermark} requires us to fill the entire table $\mT[t][g]$ for all $(t, g) \in \{(t,g) \mid 1 \leq t \leq T_{\text{max}}, 0 \leq g \leq \min \{ t-1, G_{\text{max}} \} \}$.
This allows generated texts to contain many green words \emph{locally}
because the constraint in Eq.~\eqref{eq:ns_watermark} only restricts the green words to appear above a certain number in generated texts.
To reduce this computation, we additionally impose a constraint such that the green words appear \emph{periodically} in generated texts.
Technically, it is challenging because the proportion of green words appearing in the generated text depends on its length $T$ as in Eq.~\eqref{eq:ns_watermark},
which is unknown a priori.
For instance, a generated text is typically short for a closed-ended question, and the proportion of green words needs to be large. 
By contrast, generated texts tend to be long for news articles, and the proportion of green words can be reduced.
Thus, to make green words appear periodically in the generated text, we need to estimate the length of the generated text before generating it.

To estimate the text length, we leverage the observation that the length of generated texts remains almost the same regardless of watermarks
because the text length is generally determined by the content of the generated texts, i.e., the prompt.
Inspired by this observation, we propose generating the texts without watermarks using the conventional beam search,
obtaining the length of generated text $\widehat{T}$,
and generating text with watermarks by solving the following problem:
\begin{align}
\label{eq:ns_watermark2}
    \argmax_{x_{1:T}, T} &p(x_{1:T} \mid x_{\text{prompt}}) \\
    \;\; \text{s.t.} \;\; &\frac{|x_{1:T}|_\mathrm{G}}{T-1} \geq \gamma + Z \sqrt{\frac{\gamma (1-\gamma)}{T-1}}, \nonumber \\
    &\left| \frac{|x_{1:t}|_\mathrm{G}}{t-1} - \min \left\{1, \gamma + Z \sqrt{\frac{\gamma (1-\gamma)}{\widehat{T}-1}} \right\}  \right| \leq \frac{\alpha}{t-1}
    \;\; \text{or} \;\; |x_{1:t-1}|_\mathrm{G} \geq G_{\text{max}} \;\; (t = 2, \cdots, T),
    \nonumber
\end{align}
where $\alpha \geq 1$ denotes a hyperparameter that controls the approximation rate.
Intuitively, the second inequality makes green words appear periodically
and the last inequality verifies whether $G_{\text{max}}$ green words appear before the first $t-1$ words.
As explained in the previous section, if $G_{\text{max}}$ green words appear after generating $t-1$ words,
the number of green words added in the remaining texts needs not be counted anymore.
Thus, we only need to impose the last inequality on the generated text until $G_{\text{max}}$ green words appear.
We show the visual explanation in Fig.~\ref{fig:alg2}.
Owing to this additional constraint, we do not need to fill the table $\mT[t][g]$ for all $(t, g) \in \{(t,g) \in \mathbb{N}^2 \mid 1 \leq t \leq T_{\text{max}}, 0 \leq g \leq \min \{ t-1, G_{\text{max}} \} \}$.
We only need to fill the table $\mT[t][g]$ for $(t,g)$ that satisfies the conditions in Eq.~\ref{eq:ns_watermark2} (i.e., the colored area in Fig.~\ref{fig:alg2}).
Subsequently, the time complexity can be reduced to $\mathcal{O}(\alpha k T_{\text{max}})$.
We show the pseudo-code in Sec.~\ref{sec:pseudo_code}.

\subsection{Robustness to Post-editing Attack}
\label{sec:robustness1}
In the previous sections, we proposed the watermarking methods that impose the minimum constraint to detect LLM-generated texts. 
However, due to the minimality of the constraint, the NS-Watermark can be removed from the generated texts by replacing only one green word with a red word. 
To make the watermarks robust against such editing, we can tighten the constraint as follows:
\begin{align}
\label{eq:robust_ns_watermark}
    \argmax_{x_{1:T}, T} p(x_{1:T} \mid x_{\text{prompt}}) 
    \;\; \text{s.t.} \;\; \frac{|x_{1:T}|_\mathrm{G}}{T-1} \geq \gamma + \beta + Z \sqrt{\frac{\gamma (1-\gamma)}{T-1}}, 
\end{align}
where $\beta \geq 0$ is a hyperparameter that controls the robustness.
Owing to the constraint in Eq.~\eqref{eq:robust_ns_watermark}, 
the z-score of the generated texts exceeds $Z$ even if $\beta (T-1)$ green words are replaced with red words,
and we can identify them as the texts generated by LLMs.
Moreover, the constraint in Eq.~\eqref{eq:robust_ns_watermark} is also the minimum constraint required to be imposed on the generated texts such that the z-score exceeds $Z$ after $50\beta \%$ words are replaced.
In Sec.~\ref{sec:robustness}, we experimentally evaluate the trade-off between text quality and robustness against the post-editing attack,
demonstrating that the NS-Watermark can achieve the better trade-off than the Soft-Watermark.
\begin{figure*}[b!]
    \centering
    \begin{subfigure}{0.315\hsize}
        \includegraphics[width=\hsize]{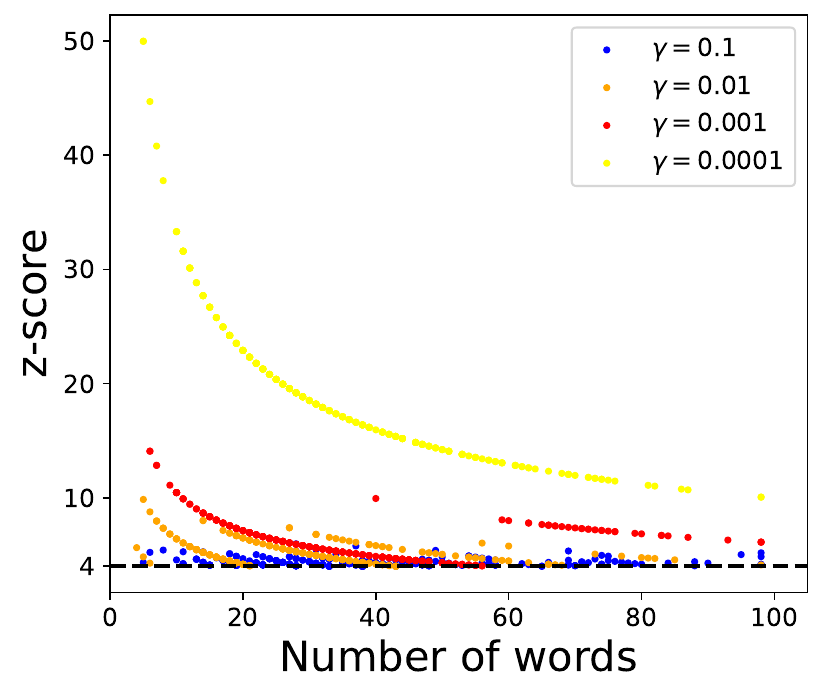}
        \caption{NS-Watermark}
    \end{subfigure}
    \begin{subfigure}{0.325\hsize}
        \includegraphics[width=\hsize]{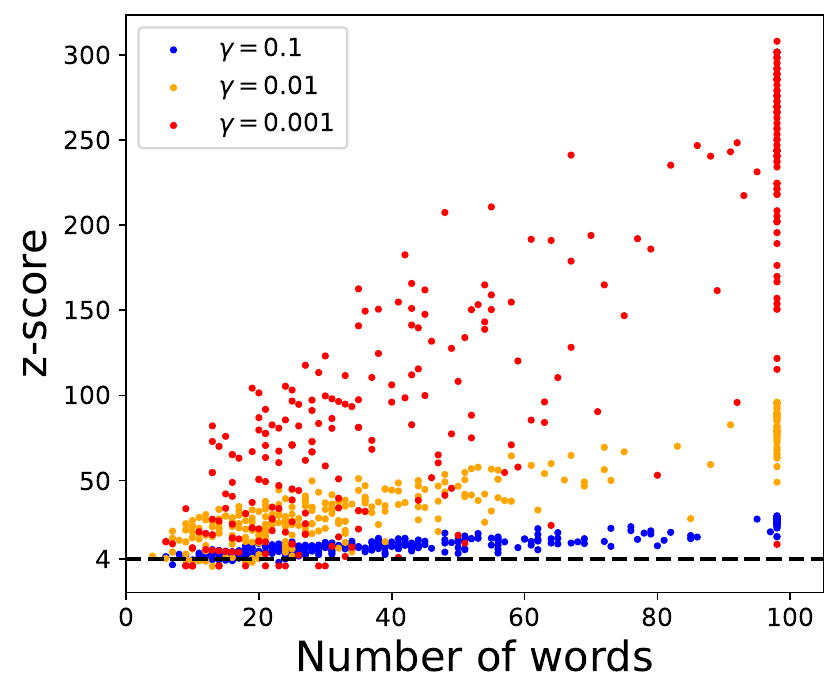}
        \caption{Soft-Watermark}
    \end{subfigure}
    \begin{subfigure}{0.325\hsize}
        \includegraphics[width=\hsize]{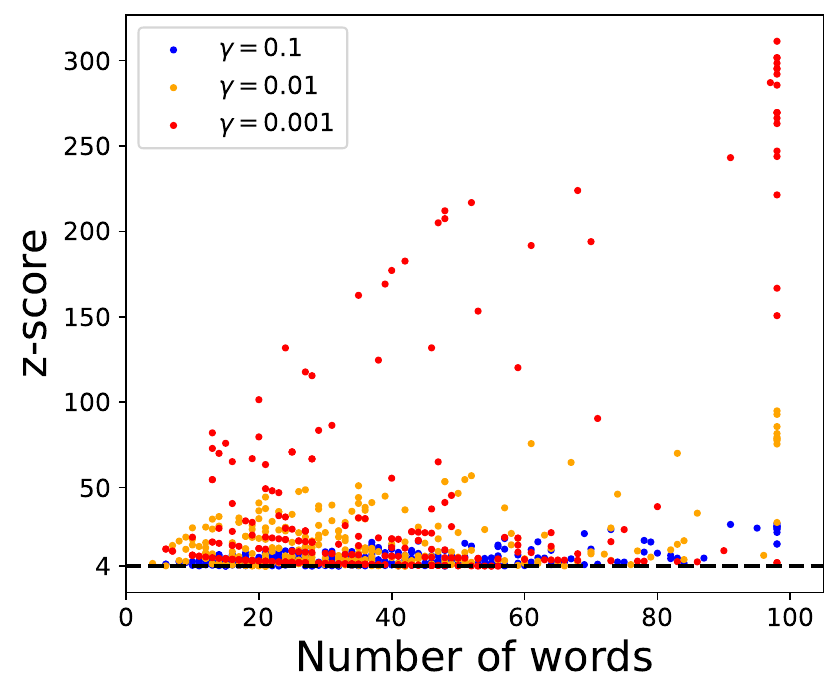}
        \caption{Adaptive Soft-Watermark}
    \end{subfigure}
    \vskip - 0.05 in
    \caption{Relationships between z-score and the length of generated texts. We used the validation datasets of WMT'16 En$\rightarrow$De. For each $\gamma$, we tuned the hyperparameter $\delta$ of the Soft-Watermark by increasing $4, 6, 8, \cdots$ and selecting the smallest value such that the FNR becomes less than $5\%$. We omit the results of the Soft-Watermark and Adaptive Soft-Watermark for $\gamma=0.0001$ because the z-scores become too large. Full results are deferred to Sec.~\ref{sec:visualization}.}
    \label{fig:z_score}
    \vskip - 0.05 in
\end{figure*}

\section{Experiments}
\label{sec:experiments}

\subsection{Comparison Methods}
In the following sections, we evaluate the following three methods:
(1) The Soft-Watermark \citep{kirchenbauer2023watermark},
which generates texts such that almost all words contained in the texts become green words
by increasing the probability that green words appear.
A hyperparameter $\delta \geq 0$ is a positive offset for the probability that green words appear.
When $\delta$ is set to a larger value, more green words appear in the generated texts.
(2) The NS-Watermark, which generates texts containing the minimum number of green words to detect LLM-generated texts, unlike the Soft-Watermark.
(3) The Adaptive Soft-Watermark, a simple extension of the Soft-Watermark.
The original Soft-Watermark uses the same hyperparameter $\delta$ for all texts,
and the proportion of green words contained in a generated text is almost constant regardless of text length.
Thus, the probability offset $\delta$ used by the Soft-Watermark increases the number of green words more than necessary to detect LLM-generated texts, especially for long texts.
We improve the Soft-Watermark such that $\delta$ is tuned for each text, which we refer to as the \emph{Adaptive Soft-Watermark}.
Specifically, the Adaptive Soft-Watermark finds $\delta$ by binary search and 
uses $\delta$ such that the z-score is minimum and exceeds the threshold $Z$.
We present the pseudo-code in Sec.~\ref{sec:pseudo_code}.
In this section, we did not evaluate Hard-Watermark since it has the very serious problem of not being able to generate common phrases or proper nouns, e.g., Barack Obama, when Obama is a red word.
Furthermore, we did not evaluate existing watermarking methods that focused on the ``undetectability'' of watermarks \citep{christ2023undetectable,kuditipudi2023robust}. Our paper focused on minimizing the text quality degradation caused by watermarking. These goals were orthogonal, and it is one of the interesting directions to design a novel watermarking method to achieve both properties.

\begin{table}[t]
\centering
\vskip -0.1 in
\caption{BLEU scores and detection accuracy with NLLB-200-3.3B and WMT. For the NS-Watermark, we set $\alpha$ to one and $\beta$ to zero. The best values among the watermarking methods are highlighted in bold.}
\label{table:translation}
\begin{tabular}{lcccc}
\toprule
 & \multicolumn{2}{c}{En $\rightarrow$ De} & \multicolumn{2}{c}{De $\rightarrow$ En}  \\
 & BLEU $\uparrow$ & FNR $\downarrow$ / FPR $\downarrow$ & BLEU $\uparrow$ & FNR $\downarrow$ / FPR $\downarrow$ \\
\midrule
w/o Watermark                            & $36.4$ & n.a. & $42.6$ & n.a. \\
Soft-Watermark                           &  $5.2$ & $3.0 \%$ / $0.4 \%$ &  $7.5$ & $3.3 \%$ / $0.5 \%$ \\
Adaptive Soft-Watermark \hspace{-15pt}   & $20.5$ & $\textbf{0.0 \%}$ / $2.6 \%$ & $20.6$ & $\textbf{0.0 \%}$ / $1.9 \%$ \\
NS-Watermark                             & $\textbf{32.7}$ & $\textbf{0.0 \%}$ / $\textbf{0.3 \%}$ & $\textbf{38.2}$ & $\textbf{0.0 \%}$ / $\textbf{0.0 \%}$ \\
\bottomrule
& & & & \\
\toprule
 & \multicolumn{2}{c}{En $\rightarrow$ Fr} & \multicolumn{2}{c}{Fr $\rightarrow$ En}  \\
 & BLEU $\uparrow$ & FNR $\downarrow$ / FPR $\downarrow$ & BLEU $\uparrow$ & FNR $\downarrow$ / FPR $\downarrow$ \\
\midrule
w/o Watermark                            & $42.6$ & n.a. & $40.8$ & n.a. \\
Soft-Watermark                           &  $9.6$ & $5.4 \%$ / $\textbf{0.3 \%}$  &  $7.6$ & $3.6 \%$ / $0.6 \%$ \\
Adaptive Soft-Watermark \hspace{-15pt}   & $23.3$ & $\textbf{0.0 \%}$ / $2.2 \%$  & $19.5$ & $\textbf{0.0 \%}$ / $2.8 \%$ \\
NS-Watermark                             & $\textbf{38.8}$ & $\textbf{0.0 \%}$ / $\textbf{0.3 \%}$  & $\textbf{36.8}$ & $\textbf{0.0 \%}$ / $\textbf{0.1 \%}$ \\
\bottomrule
\end{tabular}
\end{table}
\setlength{\textfloatsep}{10pt}

\subsection{Machine Translation}
\label{sec:machine_translation}
\textbf{Experimental Setting.}
We evaluate the effectiveness of the NS-Watermark on machine translation tasks.
We used NLLB-200-3.3B model \citep{nllbteam2022no} with the test dataset of WMT'14 French (Fr) $\leftrightarrow$ English (En) and WMT'16 German (De) $\leftrightarrow$ English (En).
Following the prior work \citep{kirchenbauer2023watermark}, we set the hyperparameter $Z$ to $4$.
For other hyperparameters, we split the data into the validation and test datasets with $10/90$ ratio and used the validation dataset to tune.
For the NS-Watermark, we selected $\gamma$ with the best BLEU score \citep{papineni2002bleu} on the validation dataset using a grid search.
The z-score of texts generated by the NS-Watermark is guaranteed to be greater than or equal to $Z$,
and the FNR of the NS-Watermark becomes exactly $0 \%$ for any hyperparameters.
By contrast, the z-score of text generated by the Soft-Watermark and Adaptive Soft-Watermark is not guaranteed to be greater than or equal to $Z$.
Thus, to fairly compare the NS-Watermark with the Soft-Watermark and Adaptive Soft-Watermark, we selected the hyperparameters of these methods with the best BLEU score while achieving more than $95 \%$ FNR in the validation dataset using a grid search.
See Sec.~\ref{sec:hyperparameter} for more detailed hyperparameter settings.

\textbf{Results.}\footnote{Note that \citet{kirchenbauer2023watermark} evaluated text quality and detection accuracy with only texts consisting of approximately $200$ words, while we evaluated them with all generated texts, which contain both short and long texts.}
Table~\ref{table:translation} shows that the NS-Watermark outperforms the Soft-Watermark and Adaptive Soft-Watermark in terms of both text quality and detection accuracy.
For all datasets, the Soft-Watermark significantly degraded the BLEU scores.
The Adaptive Soft-Watermark improved the BLEU scores by tuning $\delta$ for each text,
although it still achieved lower BLEU scores than the generated texts without watermarks.
By contrast, the NS-Watermark outperformed the Soft-Watermark by approximately $30$ BLEU scores
and achieved competitive BLEU scores with the conventional beam search without watermarks.
Moreover, the NS-Watermark can achieve $100.0 \%$ TPR because the NS-Watermark is guaranteed to insert a sufficient number of green words into the generated texts.

\textbf{Analysis of z-score and Length of Texts.}
Figure~\ref{fig:z_score} shows the z-scores and length of generated texts.
We also plot the relationship between the number of green words and the text length in Sec.~\ref{sec:visualization}.
In the Soft-Watermark, the z-score increased as generated texts became longer.
As discussed in Sec.~\ref{sec:linear},
the proportion of green words can be reduced as the length of generated texts increases.
However, the Soft-Watermark cannot change the proportion of green words adaptively to the text length,
resulting in generating texts with unnecessarily many green words for long texts.
The Adaptive Soft-Watermark mitigates this problem by tuning $\delta$ for each text,
although its z-score still increases as the length increases.
By contrast, the NS-Watermark can change the proportion of green words adaptively to the length,
and the z-score does not increase even if the text length increases.
Thus, the NS-Watermark imposes the minimum constraint to make the z-score of generated texts exceed the threshold $Z(=4)$
and can generate more natural and higher-quality texts than the Soft-Watermark and Adaptive Soft-Watermark.

\begin{figure}[t!]
\centering
\begin{minipage}[b]{0.36\columnwidth}
    \centering
    \includegraphics[height=5cm]{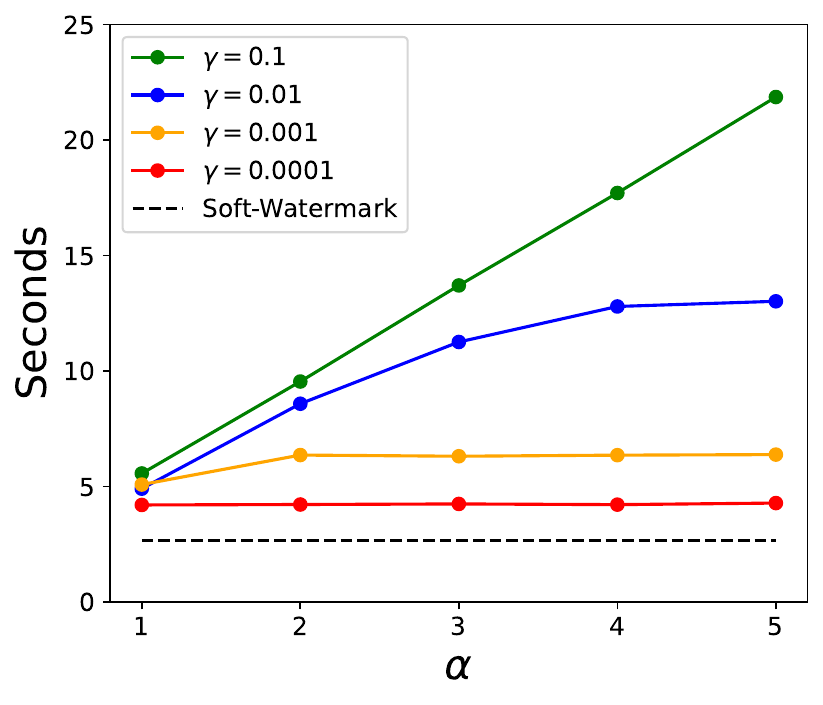}
    \caption{Time required to generate a text when varying $\alpha$.}
    \label{fig:speed}
\end{minipage}
\hfill
\begin{minipage}[b]{0.62\columnwidth}
    \centering
    \includegraphics[height=5cm]{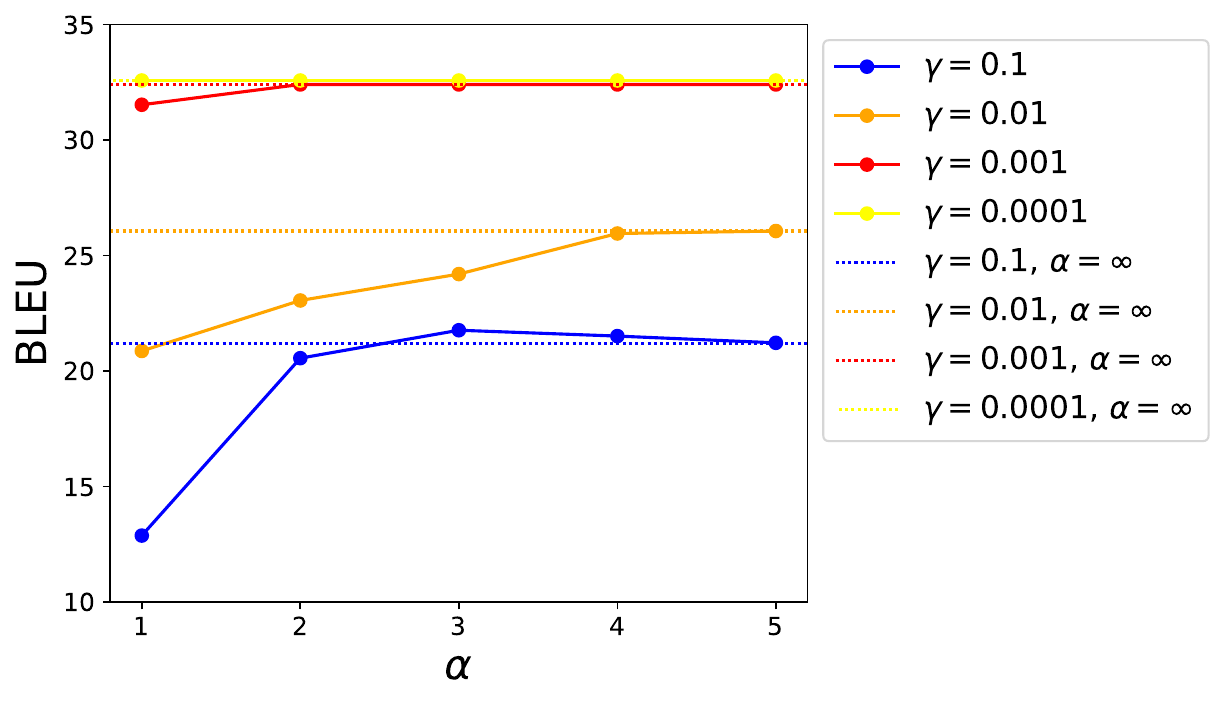}
    \caption{Text quality when varying $\alpha$. We used the validation dataset of WMT'16 En$\rightarrow$De.}
    \label{fig:alpha}
\end{minipage}
\end{figure}

\textbf{Analysis of Approximation Rate $\alpha$.}
In the above experiments, we demonstrated that the NS-Watermark outperforms the Soft-Watermark and Adaptive Soft-Watermark when $\alpha$ is set to one.
As we explained in Sec.~\ref{sec:linear}, the text quality can be improved using the larger $\alpha$.
In this section, we analyze the sensitivity of $\alpha$ on text quality.
Figure~\ref{fig:alpha} shows the results when varying $\alpha$.
The results indicate that when $\gamma$ is large, the BLEU scores increase with $\alpha$,
but when $\gamma$ is small, the BLEU scores are almost the same.
This is because more green words need to be contained in the generated texts when $\gamma$ is large.
Therefore, the larger $\gamma$, the greater the influence of the approximation,
and we need to use large $\alpha$ to generate high-quality texts.
Fortunately, because the NS-Watermark achieved the best BLEU score when $\gamma$ was small,
we can use the small $\alpha$ without degrading text quality much in practice.

\textbf{Running Time.}
Fig.~\ref{fig:speed} shows the running time when varying $\alpha$.
To measure the running time, we used the validation dataset of WMT'16 En$\rightarrow$De and reported the average running time.
For $\gamma=0.1, 0.01, 0.001, 0.0001$, $G_{\text{max}}$ is $22, 5, 2$, and $1$, respectively.
The results indicate that the running time increases as $\alpha$ increases when $\gamma$ is large.
This result was consistent with the time complexity of Alg.~\ref{alg:fast_ns_watermark}, which we discussed in Sec.~\ref{sec:linear}.
Then, when $\gamma$ was small, the running time was almost the same even if $\alpha$ increased.
This is because when $\gamma$ is small, $G_{\text{max}}$ is small, and the range in the table $\mT[t][g]$ where we need to fill in Alg.~\ref{alg:ns_watermark} is small.
Thus, the range in $\mT[t][g]$ where we need to calculate does not increase even if $\alpha$ increases when $\gamma$ is small.

\subsection{Natural Language Generation}
\textbf{Experimental Setting.}
Next, we compare the NS-Watermark and the Soft-Watermark in terms of perplexity (PPL).
We used LLaMA-7B model \citep{touvron2023llama} with the subsets of C4, realnews-like dataset \citep{raffel2020exploring}.
Based on the prior work \citep{kirchenbauer2023watermark},
we split each text and used the first $90 \%$ of words as the prompt to infer the remaining $10 \%$ of words using LLMs.
We regarded the last $10 \%$ of words contained in the data as the text written by humans
and compared the NS-Watermark with the Soft-Watermark and Adaptive Soft-Watermark in terms of PPL and detection accuracy.
Then, we set the hyperparameter $Z$ to $4$.
To tune hyperparameters, we split the dataset into validation and test datasets with $10/90$ ratio.
As in the previous section, we selected the hyperparameters of the Soft-Watermark and Adaptive Soft-Watermark with the best PPL while achieving more than $95 \%$ FNR in the validation dataset using a grid search.
See Sec.~\ref{sec:hyperparameter} for more detailed hyperparameter settings.

\textbf{Results.}
The results are listed in Table~\ref{table:ppl_vs_accuracy}.
The results were consistent with those presented in Sec.~\ref{sec:machine_translation}.
The NS-Watermark can outperform the Soft-Watermark and Adaptive Soft-Watermark in terms of text quality and detection accuracy.

\subsection{Robustness to Post-editing Attack}
\label{sec:robustness}

In this section, we analyze the trade-off between text quality degradation and robustness to post-editing, demonstrating that the NS-Watermark can achieve a better trade-off than the Soft-Watermark.

\textbf{Experimental Setting.}
To simulate the post-editing attack, we randomly select $\epsilon \%$ words in the generated texts and replace them with random words.
We then analyzed how much FNR increases after the editing.
We used LLaMA-7B model and the validation dataset of C4.
For the NS-Watermark, we showed the results when varying $\beta \in \{0,0.05,0.1,0.2\}$.
For other hyperparameters, we used the values shown in Sec.~\ref{sec:hyperparameter}.

\textbf{Results.}
Figure \ref{fig:attack} shows the relationship between text quality and robustness against post-editing of the NS-Watermark and Soft-Watermark with various hyperparameters.
By comparing the results of the NS-Watermark and Soft-Watermark with the same level of PPL, 
we observe that the NS-Watermark achieved a smaller FNR than the Soft-Watermark.
Thus, the NS-Watermark can achieve a better trade-off between text quality and robustness than the Soft-Watermark.
In Sec.~\ref{sec:robustness_apendix}, we also compared the NS-Watermark with the Adaptive Soft-Watermark and demonstrated that the NS-Watermark can achieve a better trade-off than the Adaptive Soft-Watermark.
\begin{table}[t]
\vskip - 0.05 in
\centering
\caption{Text quality and detection accuracy with LLaMA-7B and C4 dataset. For the NS-Watermark, we set $(\alpha, \beta)=(1,0)$. The best values among the watermarking methods are shown in bold.}
\label{table:ppl_vs_accuracy}
\small
\begin{tabular}{lcc}
\toprule
& PPL $\downarrow$ & FNR $\downarrow$ / FPR $\downarrow$ \\
\midrule
w/o Watermark              & $1.85$ & n.a. \\
Soft-Watermark             & $6.25$ & $2.8 \%$ / $\textbf{0.1 \%}$ \\
Adaptive Soft-Watermark    & $2.48$ & $0.2 \%$ / $0.8 \%$ \\
NS-Watermark               & $\textbf{1.92}$ & $\textbf{0.0 \%}$ / $0.3 \%$ \\
\bottomrule
\end{tabular} 
\end{table}
\begin{figure*}[t]
    \centering
    \vskip - 0.2 in
    \begin{subfigure}{0.45\hsize}    
        \includegraphics[width=\hsize]{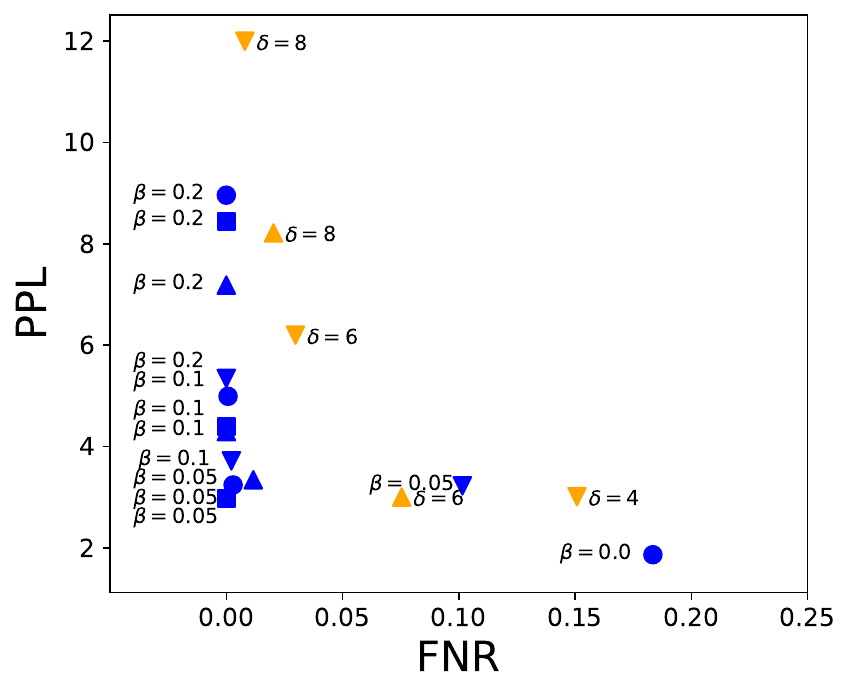}
        \vskip - 0.05 in
        \caption{$\epsilon=0.1$}
    \end{subfigure}
    \begin{subfigure}{0.45\hsize}    
        \includegraphics[width=\hsize]{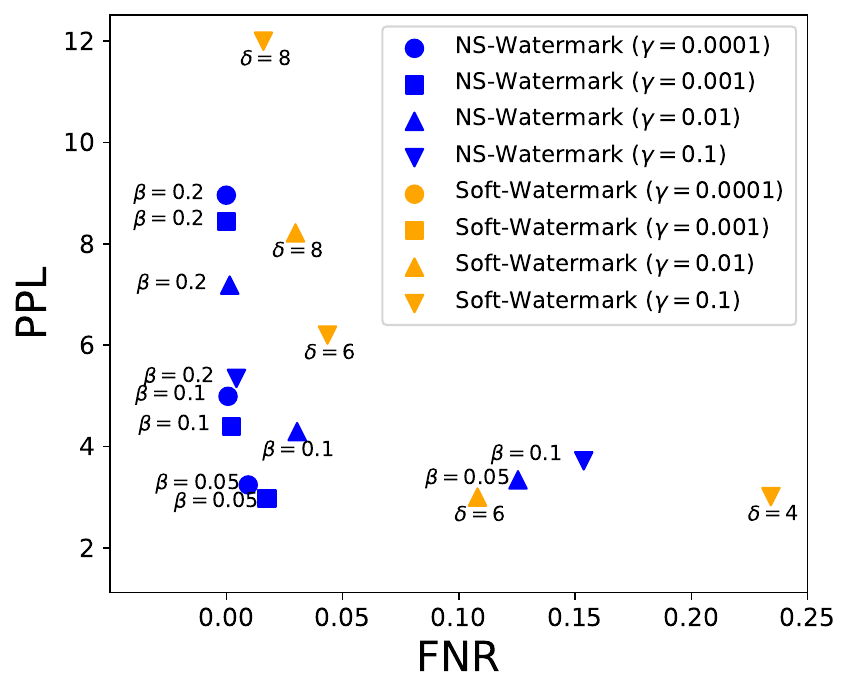}
        \vskip - 0.05 in
        \caption{$\epsilon=0.2$}
    \end{subfigure}
    \caption{Trade-off between text quality and robustness against post-editing. To make the figure more readable, the results with FNR greater than $25\%$ were omitted. The method with the point in the lower left corner is the superior method. Surprisingly, the NS-Watermark is generally more robust against the post-editing than the Soft-Watermark even with a small offset $\beta=0.05$.}
    \label{fig:attack}
\end{figure*}

\section{Related Work}
\label{sec:related_work}
\textbf{Watermarking Methods.}
Watermarking methods detect LLM-generated texts by inserting imperceptible information into generated texts.
Watermarking methods have been extensively studied for images and audio \citep{luo2020distortion,liu2023dear}.
However, due to the discrete structure of language, watermarking methods for natural language have been more challenging than those for images and audio.
Recently, \citet{kirchenbauer2023watermark} proposed the first practical watermarking method for LLMs.
\citet{kuditipudi2023robust} have extended it and proposed methods that are robust against post-editing,
and \citet{christ2023undetectable} proposed undetectable methods.
These methods skew the distributions of generated texts (e.g., the ratio of green and red words) and detect LLM-generated texts using statistical hypothesis testing.
One advantage of watermarking methods is their high detection accuracy.
Furthermore, thanks to statistical hypothesis testing, the FPR can be explicitly adjusted by the hyperparameter.
However, because watermarking methods need to modify generated texts, generated texts are often of low quality.
Our experiments indicated that the existing methods underestimated text quality degradation caused by watermarking, 
and the NS-Watermark differs from them in that it aims to minimize text-quality degradation.


\textbf{Post-hoc Detection Methods.}
As an alternative approach, post-hoc detection methods have been proposed \citep{zellers2019defending,gehrmann2019gltr,tian2023gptzero,mitchell2023detectgpt}.
\citet{zellers2019defending} and \citet{tian2023gptzero} proposed training additional models to detect LLM-generated texts.
\citet{mitchell2023detectgpt} also found that LLM-generated texts tend to be texts at which the curvature of the LLMs' log-likelihood becomes negative 
and demonstrated that those can be identified without training additional models.
These post-hoc methods do not degrade text quality because they do not modify generated texts.
However, post-hoc methods are inferior to watermarking methods in detection accuracy \citep{krishna2023paraphrasing}.

\section{Conclusion}
In this study, we proposed the NS-Watermark 
for inserting watermarks into generated texts with minimum text quality degradation.
Specifically, we proposed the NS-Watermark that imposes the minimum constraint required to detect LLM-generated texts on generated texts.
We conducted the experiments on various tasks,
showing that the NS-Watermark can achieve $0 \%$ false negative rate with negligible text quality degradation.
Furthermore, the experimental results demonstrated that the NS-Watermark improves the quality-detectability trade-off for post-editing attacks.

\section{Limitations}
The NS-Watermark can minimize text quality degradation, while it increases the time complexity.
We proposed the approximation methods to alleviate this issue, but its time complexity is still higher than that of the Soft-Watermark.
Reducing the time complexity of the NS-Watermark is one of the promising research directions.
Besides the methods that minimize text quality degradation, several papers studied undetectable watermarking methods \citep{hu2024unbiased,christ2023undetectable}.
It is also a very promising research direction to study undetectable watermarking methods with minimum text quality degradation by extending our work.

\bibliography{custom}
\bibliographystyle{tmlr}

\newpage
\appendix
\section{Necessary and Sufficient Watermark is Provably Better than Soft-Watermark}
In Sec.~\ref{sec:experiments}, we demonstrated that the Soft-Watermark imposes overly restrictive constraints on generated texts
and inserts too many green words to detect LLM-generated texts.
The Adaptive Soft-Watermark tunes $\delta$ for each text, but Sec.~\ref{sec:experiments} showed that the Adaptive Soft-Watermark remains to insert too many green words into texts.
We rigorously analyze this issue, providing the following theorem,
which shows that no matter how well the hyperparameter $\delta$ is tuned for each text, the Soft-Watermark cannot precisely control the number of green words in generated texts and generates text that contains more than the required number of green words.
\begin{theorem}[Informal]
\label{th:adaptive_soft_watermark}
If we select minimum $\delta^\star \in \mathbb{R}$ such that the z-score of the text generated by the Soft-Watermark exceeds the threshold $Z$,
the Soft-Watermark generates a text that contains more than the required number of green words with non-zero probability.
\end{theorem}
The formal theorem and its proof are presented in Sec.~\ref{sec:proof}.
Unlike the Soft-Watermark and Adaptive Soft-Watermark,
the NS-Watermark can insert the minimum number of green words into generated texts using dynamic programming
and thus can generate higher-quality texts than these methods.

\section{Proof of Theorem \ref{th:adaptive_soft_watermark}}
\label{sec:proof}

\begin{assumption}
\label{assumption:beam_size}
    The beam size is set to one.
\end{assumption}
\begin{assumption}
\label{assumption:length}
    The length of generated texts $T$ is sufficiently long.
\end{assumption}
\begin{assumption}
\label{assumption:delta}
    $\Delta=\mathbb{R}$.
\end{assumption}
\begin{assumption}
\label{assumption:gamma}
    The hyperparameter $\gamma$ is sufficiently small such that the text generated by the greedy search does not contain green words, and texts containing a single green word have a z-score greater than the threshold $Z$ for any length $T$.
\end{assumption}
\begin{assumption}
\label{assumption:logit}
    For any prompt $x_{\text{prompt}}$ and generated text $x_{1:T}$, $L_t(x_{1:T})$, defined below, is an independent and identically distributed random variable that follows the distribution $L_t (x_{1:T}) \sim p(\cdot)$:
    \begin{align*}
        L_t(x_{1:T}) &\coloneqq \text{logit} (x_t \mid x_{1:t-1}, x_{\text{prompt}}) \\
        &- \max_{x \in V^{\text{green}} (x_{t-1})} \text{logit} (x \mid x_{1:t-1}, x_{\text{prompt}}),
    \end{align*}
    where $\text{logit} (\cdot \mid \cdot)$ is the output just before the last softmax layer in LLMs.
    Furthermore, we assume that its cumulative distribution function is strictly increasing.
\end{assumption}

\begin{lemma}
\label{lemma:delta}
Let $r_{1:T}$ be the text generated by greedy search (i.e., text without watermarks).
We then select the minimum $\delta^\star \in \Delta$ such that the z-score of the text generated by the Soft-Watermark exceeds the threshold $Z$.
Under Assumptions \ref{assumption:beam_size}, \ref{assumption:length}, \ref{assumption:gamma}, and \ref{assumption:logit},
the selected hyperparameter $\delta^\star$ satisfies the following:
\begin{align*}
    \delta^\star \coloneqq \min_{\delta} \{ \delta \in \Delta \mid \delta \geq \min_t L_t(r_{1:T})\}.
\end{align*}
Furthermore, under Assumption \ref{assumption:delta}, $\delta^\star$ satisfies
\begin{align*}
    \delta^\star = \min_t L_t(r_{1:T}).
\end{align*}
\begin{proof}
If $\delta < \min_t L_t(r_{1:T})$, the text generated by the Soft-Watermark is $r_{1:T}$ and does not contain green words.
If $\delta \geq \min_t L_t(r_{1:T})$, $(\min \{ t \mid L_t(r_{1:T}) \le \delta \})$-th word becomes a green word,
and the text contains at least one green word.
Thus, we can obtain the statement.
\end{proof}
\end{lemma}

\setcounter{theorem}{0}
\begin{theorem}[Formal]
\label{th:adaptive_soft_watermark_formal}
We select the minimum $\delta^\star \in \Delta$ such that the z-score of the text generated by the Soft-Watermark exceeds the threshold $Z$.
Under Assumptions \ref{assumption:beam_size}, \ref{assumption:length}, \ref{assumption:delta}, \ref{assumption:gamma}, and \ref{assumption:logit},
the text generated by the Soft-Watermark with $\delta^\star$ contains two or more green words with probability $1-\log 2$.
\end{theorem}
\begin{proof}
Let $r_{1:T}$ be the text generated by the greedy search.
We define $c(\cdot)$ and $L_{\text{min}}$ as follows:
\begin{align*}
    c (a) &\coloneqq \int_a^{\infty} p(l) d_l, \\
    L_{\text{min}} &\coloneqq \min_t L_t(r_{1:T}).
\end{align*}
From Lemma \ref{lemma:delta}, we have $\delta^\star = \min_t L_t(r_{1:T})$ 
and $t^\star (\coloneqq \min \{ t \mid L_t(r_{1:T}) \le \delta \})$-th word becomes a green word in the text generated by the Soft-Watermark.
Then, the probability that the remaining text has no green words is $c (L_{\text{min}})^{T-t^\star}$.
Moreover, giving that the $t^\star$ has equal probabilities for $2, 3, \cdots, T$, 
the probability that the texts generated by the Soft-Watermark with $\delta^\star$ contain only one green word is given as
\begin{align}
\label{eq:prob_one_green}
    \frac{1}{T-1} \sum_{t=2}^{T} c (L_{\text{min}})^{T-t} = \frac{1 - c (L_{\text{min}})^{T-1}}{(T-1) (1 - c (L_{\text{min}}))}.
\end{align}
Next, we consider the distribution of $L_{\text{min}}$.
Now, we have
\begin{align*}
   \mathrm{Pr} (L_{\text{min}} \geq a) = c(a)^{T-1}.
\end{align*}
Thus, by substituting $a = c^{-1} (1 - \frac{s}{T-1})$, we can get
\begin{align*}
    \mathrm{Pr} (L_{\text{min}} \geq c^{-1} (1 - \frac{s}{T-1})) &= (1 - \frac{s}{T-1})^{T-1}, \\
    \mathrm{Pr} ( \frac{s}{T-1} \leq 1 - c (L_{\text{min}})) &= (1 - \frac{s}{T-1})^{T-1}.
\end{align*}
Defining $Y = (T-1) (1 - c (L_{\text{min}}))$, we obtain
\begin{align}
\label{eq:prob_z}
    \mathrm{Pr} (Y \leq s) &= 1 - (1 - \frac{s}{T-1})^{T-1} \nonumber \\
    &\xrightarrow{T \rightarrow \infty} 1 - e^{-s}.
\end{align}
By substituting the definition of $Y$, we can rewrite Eq.~\eqref{eq:prob_one_green} as follows:
\begin{align}
\label{eq:prob_one_green2}
    \frac{1}{T-1} \sum_{t=2}^{T} c (L_{\text{min}})^{T-t}
    &= \frac{1 - (1 - \frac{Y}{T-1})^{T-1}}{Y} \nonumber \\
    &\xrightarrow{T \rightarrow \infty} \frac{1 - e^{-Y}}{Y}.
\end{align}
Combining Eqs.~\eqref{eq:prob_z} and \eqref{eq:prob_one_green2}, we can obtain the statement.
\end{proof}
Assumption \ref{assumption:gamma} indicates that a single green word is sufficient to make the z-score exceed the threshold.
However, Theorem \ref{th:adaptive_soft_watermark_formal} indicates that texts generated by the Soft-Watermark contain two or more green words with non-zero probability, even if we tune the hyperparameter $\delta$ for each text.

\newpage
\section{Additional Visual Explanation}
\label{sec:additional_visual_explation}
\begin{figure*}[h]
    \centering
    \begin{subfigure}{0.61\hsize}
        \centering
        \includegraphics[width=0.49\hsize]{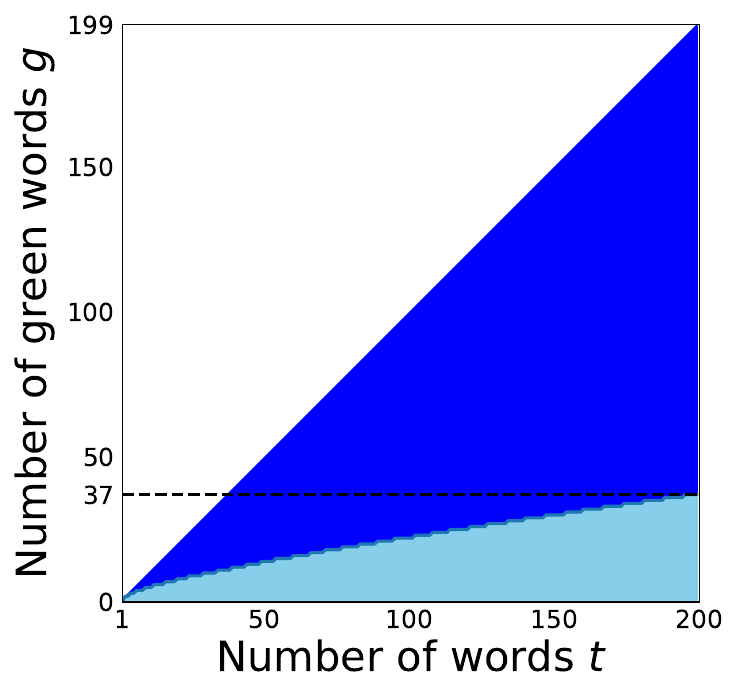}
        \includegraphics[width=0.49\hsize]{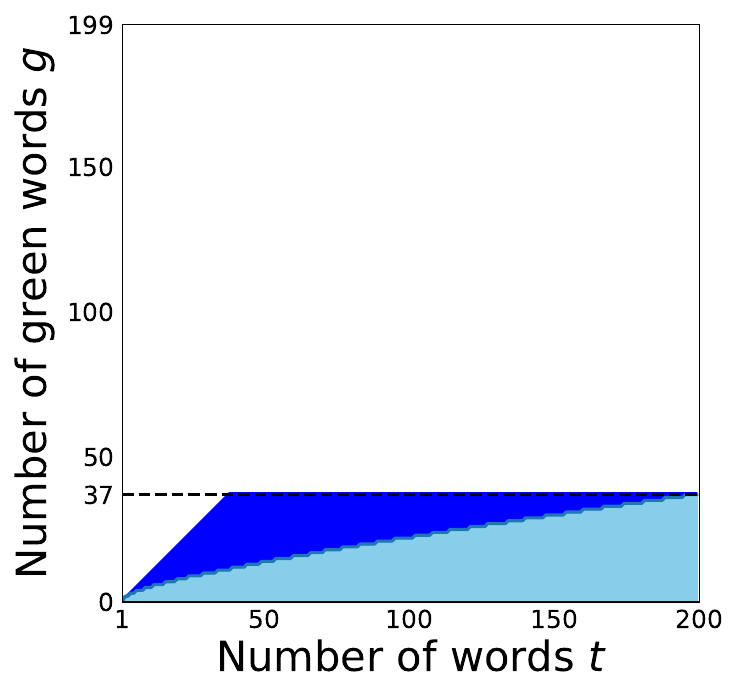}
        \caption{Illustration of Alg.~\ref{alg:ns_watermark} with $\gamma = 0.1$.}
    \end{subfigure}
    \hfill
    \begin{subfigure}{0.35\hsize}
        \centering
        \includegraphics[width=0.85\hsize]{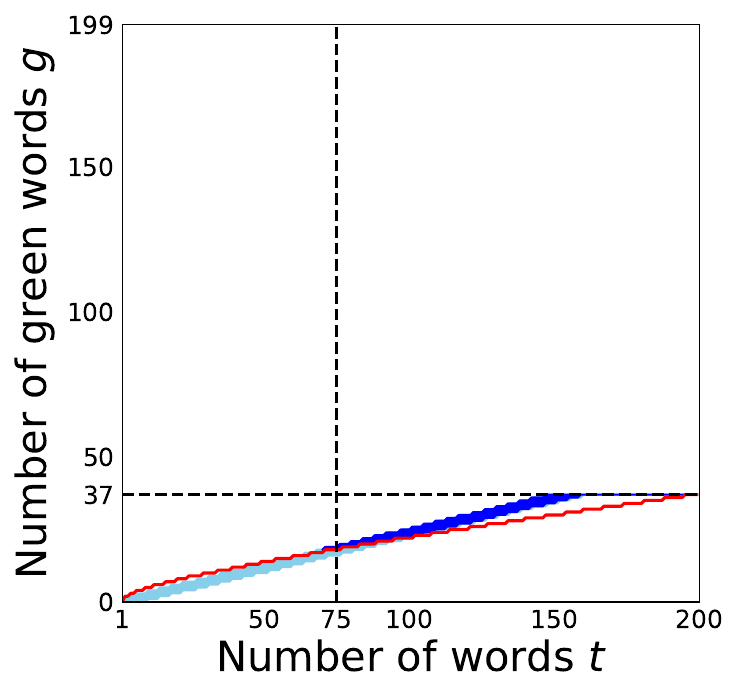}
        \caption{Illustration of Alg.~\ref{alg:fast_ns_watermark} with $\gamma = 0.1$.}
    \end{subfigure}
    \begin{subfigure}{0.61\hsize}
        \centering
        \includegraphics[width=0.49\hsize]{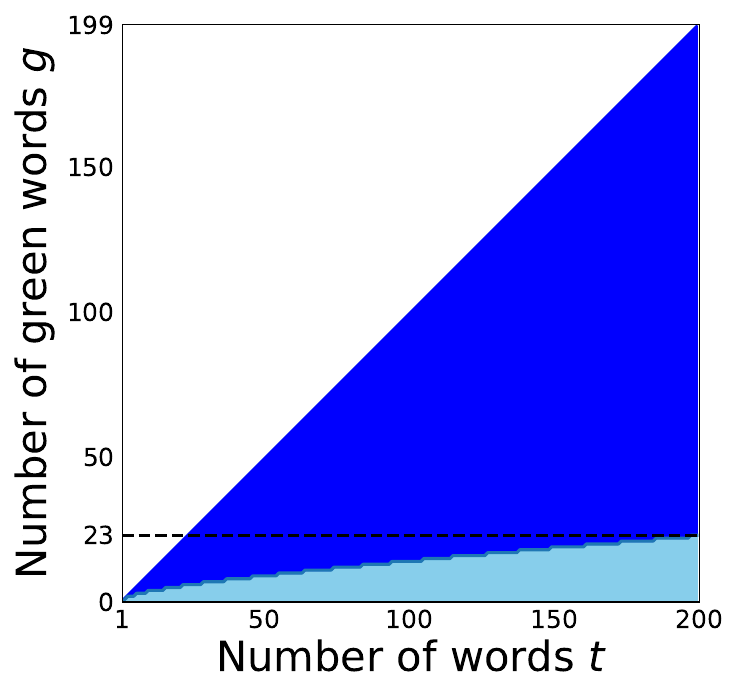}
        \includegraphics[width=0.49\hsize]{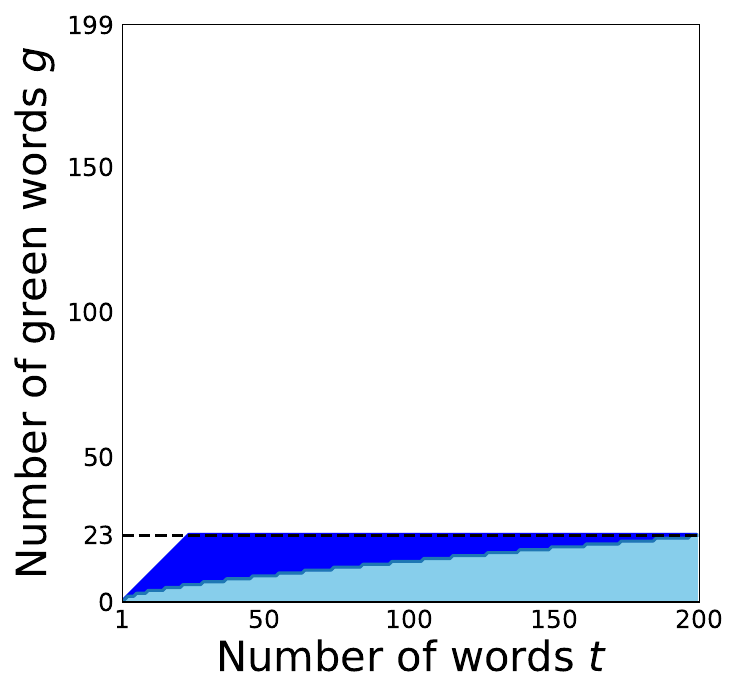}
        \caption{Illustration of Alg.~\ref{alg:ns_watermark} with $\gamma = 0.05$.}
    \end{subfigure}
    \hfill
    \begin{subfigure}{0.35\hsize}
        \centering
        \includegraphics[width=0.85\hsize]{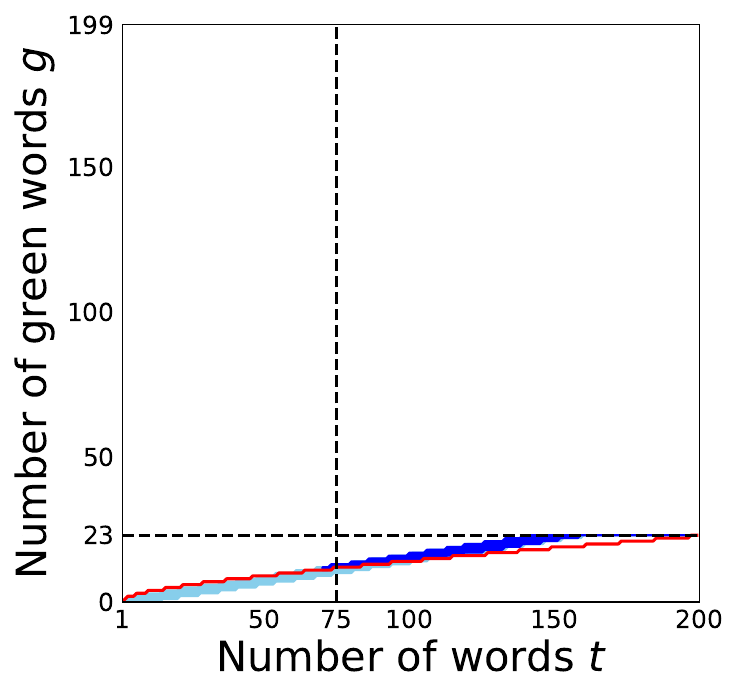}
        \caption{Illustration of Alg.~\ref{alg:fast_ns_watermark} with $\gamma = 0.05$.}
    \end{subfigure}
    \begin{subfigure}{0.61\hsize}
        \centering
        \includegraphics[width=0.49\hsize]{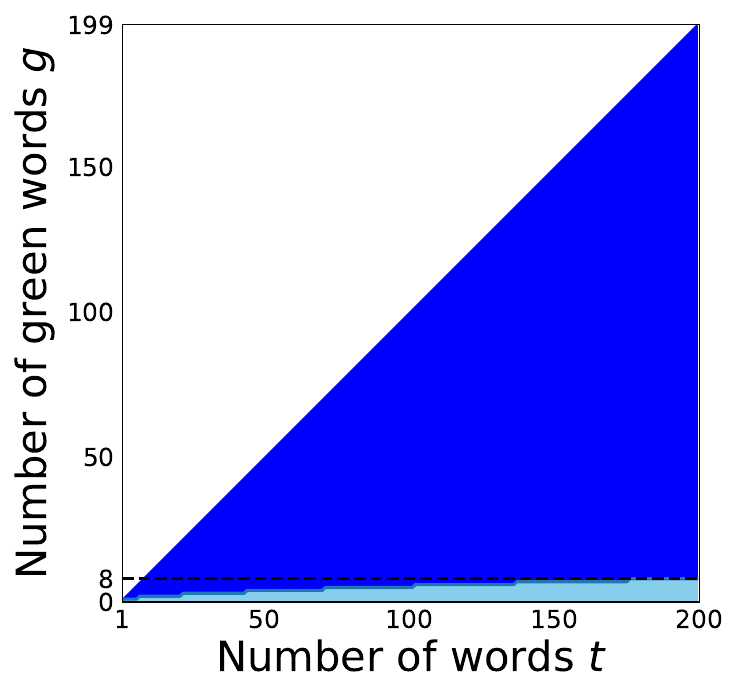}
        \includegraphics[width=0.49\hsize]{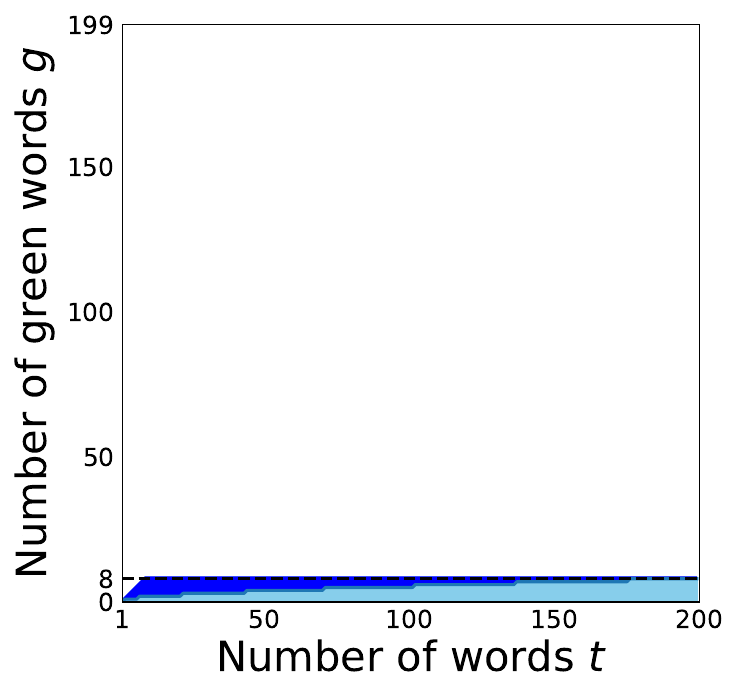}
        \caption{Illustration of Alg.~\ref{alg:ns_watermark} with $\gamma = 0.01$.}
    \end{subfigure}
    \hfill
    \begin{subfigure}{0.35\hsize}
        \centering
        \includegraphics[width=0.85\hsize]{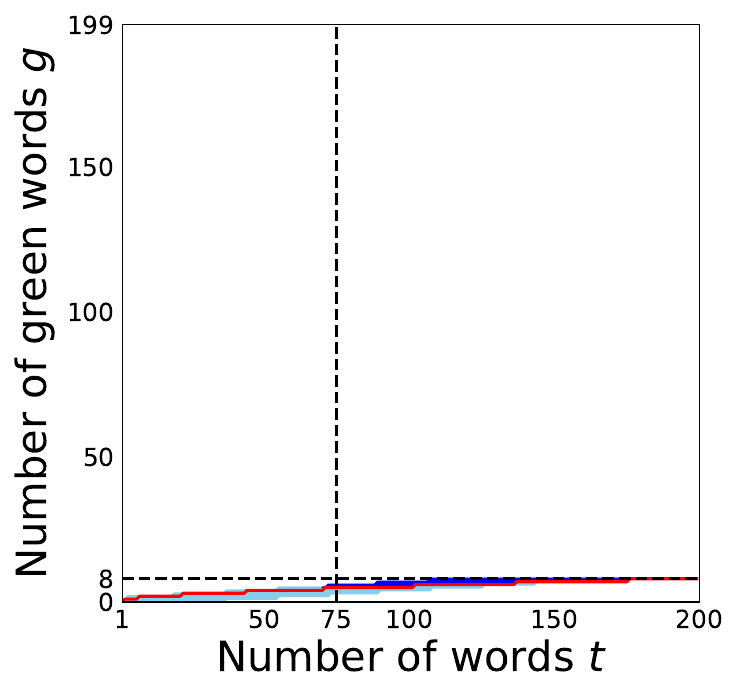}
        \caption{Illustration of Alg.~\ref{alg:fast_ns_watermark} with $\gamma = 0.01$.}
    \end{subfigure}
    \caption{Visualization of the table $\mT[t][g]$ for $T_{\text{max}}=200$, $\widehat{T}=75$, $\alpha=2$, $Z=4$, and various $\gamma$. The areas colored in blue and light blue indicate the range in $\mT[t][g]$ where we need to calculate, and the areas colored in blue indicate the range that satisfies the constraint of Eq.~\eqref{eq:ns_watermark}. The red line indicates the minimum number of green words required to satisfy the constraint. Note that in the middle and right figures, $\mT[t][G_{\text{max}}]$ does not denote the set of texts of length $t$ containing $G_{\text{max}}$ green words, but denotes the set of texts containing at least $G_{\text{max}}$ green words.}
\end{figure*}

\newpage
\section{Hyperparameter Setting}
\label{sec:hyperparameter}

In our experiments, we set the hyperparameters as follows.
All experiments were run with an A100 GPU.

\begin{table}[h!]
    \caption{Hyperparameter settings for the NS-Watermark.}
    \label{table:hyperparameter_wmt_ns}
    \centering
    \begin{tabular}{ll}
    \toprule
    Pre-trained model & NLLB-200 / LLaMA\\
    \midrule
     $k$ & $1$ \\
     $T_{\text{max}}$ & $100$ \\
     $\gamma$ & Grid search over $\{0.1, 0.01, 0.001, 0.0001 \}$. \\
    \bottomrule
    \end{tabular}
\end{table}
\begin{table}[h!]
    \caption{Hyperparameter settings for the Soft-Watermark.}
    \label{table:hyperparameter_wmt_soft}
    \centering
    \begin{tabular}{ll}
    \toprule
     Pre-trained model & NLLB-200 / LLaMA\\
    \midrule
     $k$ & $1$ \\
     $T_{\text{max}}$ & $100$ \\
     $\gamma$ & Grid search over $\{ 0.1, 0.01, 0.001, 0.0001 \}$. \\
     $\delta$ & Grid search over $\{4,6,8\}$. \\
    \bottomrule
    \end{tabular}
\end{table}

\begin{table}[h!]
    \caption{Hyperparameter settings for the Adaptive Soft-Watermark.}
    \label{table:hyperparameter_wmt_adaptive_soft}
    \centering
    \begin{tabular}{ll}
    \toprule
     Pre-trained model & NLLB-200 / LLaMA\\
    \midrule
     $k$ & $1$ \\
     $T_{\text{max}}$ & $100$ \\
     $\gamma$ & Grid search over $\{ 0.1, 0.01, 0.001, 0.0001 \}$. \\
     $\Delta$ & $\{4,6,8,10,12,14\}$ \\
     $\delta$ & Binary search over $\Delta$ for each text. \\
    \bottomrule
    \end{tabular}
\end{table}

\newpage
\section{Additional Experimental Reuslts}

\subsection{Robustness to Post-editing Attack}
\label{sec:robustness_apendix}
In this section, we compare the NS-Watermark with the Adaptive Soft-Watermark regarding a trade-off between text quality and robustness.
The Adaptive Soft-Watermark tunes the hyperparameter $\delta$ for each text such that the z-score is the minimum and exceeds the threshold $Z$.
Thus, the Adaptive Soft-Watermark may be removed from the generated texts by replacing only one green word with a red word.
To fairly compare the NS-Watermark with the Adaptive Soft-Watermark, we modify the Adaptive Soft-Watermark as in Sec.~\ref{sec:robustness1} such that the z-score exceeds the threshold $Z$ after replacing $50 \beta \%$ words,
where $\beta$ is the hyperparameter that controls the robustness against post-editing.
The pseudo-code is presented in Alg.~\ref{alg:robust_adaptive_soft_watermark}.
Then, we compare the NS-Watermark with the Adaptive Soft-Watermark,
demonstrating that the NS-Watermark can generate more natural texts than the Adaptive Soft-Watermark when $\beta$ is the same.

\begin{table}[h]
\centering
\caption{Trade-off between text quality and the hyperparameter $\beta$. The best values are highlighted in bold.}
\label{table:robustness}
\begin{tabular}{cccc}
\toprule
$\beta$                 & Method                        & PPL $\downarrow$     & FNR $\downarrow$ / FPR $\downarrow$ \\ 
\midrule
\multirow{2}{*}{$0.0$}  & Adaptive Soft-Watermark       &          $2.48$  & $0.2 \%$ / $0.8 \%$ \\ 
                        & NS-Watermark                  &  $\textbf{1.92}$ & $\textbf{0.0 \%}$ / $\textbf{0.3 \%}$ \\ 
\midrule
\multirow{2}{*}{$0.05$} & Adaptive Soft-Watermark       &           $3.43$ &  $0.2 \%$ / $\textbf{0.8 \%}$       \\ 
                        & NS-Watermark                  &  $\textbf{3.37}$ &  $\textbf{0.0 \%}$ / $\textbf{0.8 \%}$       \\ 
\midrule
\multirow{2}{*}{$0.1$}  & Adaptive Soft-Watermark       &           $4.28$ &  $0.1 \%$ / $\textbf{0.1 \%}$      \\ 
                        & NS-Watermark                  &  $\textbf{3.76}$ &  $\textbf{0.0 \%}$ / $\textbf{0.1 \%}$      \\ 
\midrule
\multirow{2}{*}{$0.2$}  & Adaptive Soft-Watermark       &           $6.02$ &  $0.1 \%$ / $\textbf{0.1 \%}$      \\ 
                        & NS-Watermark                  &  $\textbf{5.42}$ &  $\textbf{0.0 \%}$ / $\textbf{0.1 \%}$      \\ 
\bottomrule
\end{tabular}
\end{table}
\textbf{Experimental Setting.} 
We used LLaMA-7B with C4 dataset. 
We set $\alpha$ to one for the NS-Watermark and tuned the hyperparameter $\gamma$ using the validation dataset for the NS-Watermark and Adaptive Soft-Watermark.
Then, we show the results when varying $\beta \in \{0, 0.05, 0.1, 0.2\}$.
See Sec~\ref{sec:hyperparameter} for more detailed hyperparameter settings.

\textbf{Results.} Table \ref{table:robustness} lists the results.
Comparing the PPL of the NS-Watermark and Adaptive Soft-Watermark with the same $\beta$, the NS-Watermark consistently outperforms the Adaptive Soft-Watermark.
Thus, the NS-Watermark achieves a better trade-off between text quality and robustness against post-editing than the Adaptive Soft-Watermark.

\subsection{Visualization}
\label{sec:visualization}
Figures \ref{fig:z_score_full} and \ref{fig:num_green} show the relationship between z-score and text length and that between the number of green words and text length, respectively.
\begin{figure*}[h]
    \centering
    \begin{subfigure}{0.315\hsize}
        \includegraphics[width=\hsize]{pic/ns_watermark_wmt_en-de.pdf}
        \caption{NS-Watermark}
    \end{subfigure}
    \begin{subfigure}{0.325\hsize}
        \includegraphics[width=\hsize]{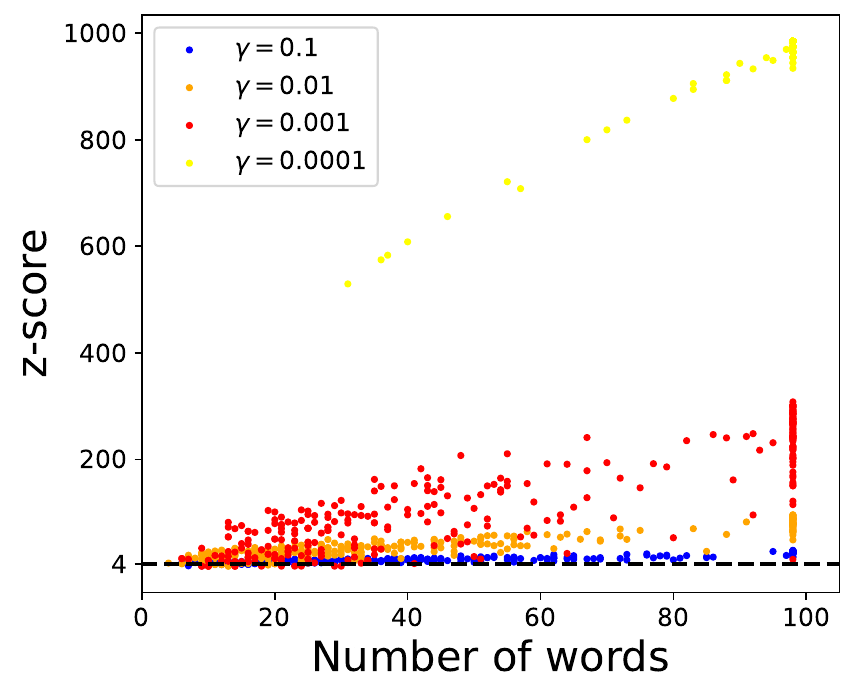}
        \caption{Soft-Watermark}
    \end{subfigure}
    \begin{subfigure}{0.325\hsize}
        \includegraphics[width=\hsize]{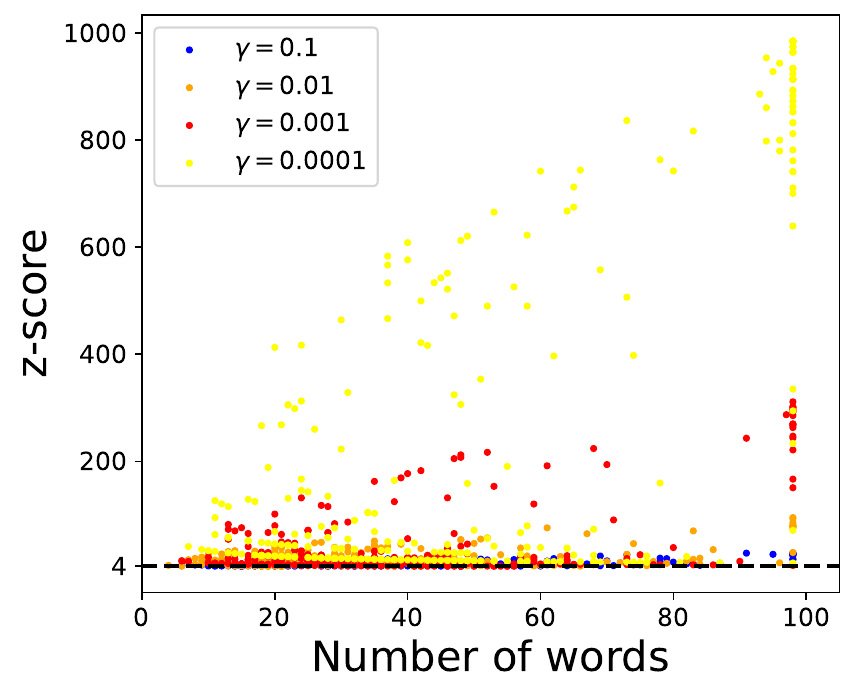}
        \caption{Adaptive Soft-Watermark}
    \end{subfigure}
    \vskip - 0.05 in
    \caption{Relationships between z-score and the length of generated texts. We used the validation datasets of WMT'16 En$\rightarrow$De. For each $\gamma$, we tuned the hyperparameter $\delta$ of the Soft-Watermark by increasing $4, 6, 8, \cdots$ and selecting the smallest value such that the FNR becomes less than $5\%$.}
    \label{fig:z_score_full}
    \vskip - 0.05 in
\end{figure*}
\begin{figure*}[h]
    \centering
    \begin{subfigure}{0.315\hsize}
        \includegraphics[width=\hsize]{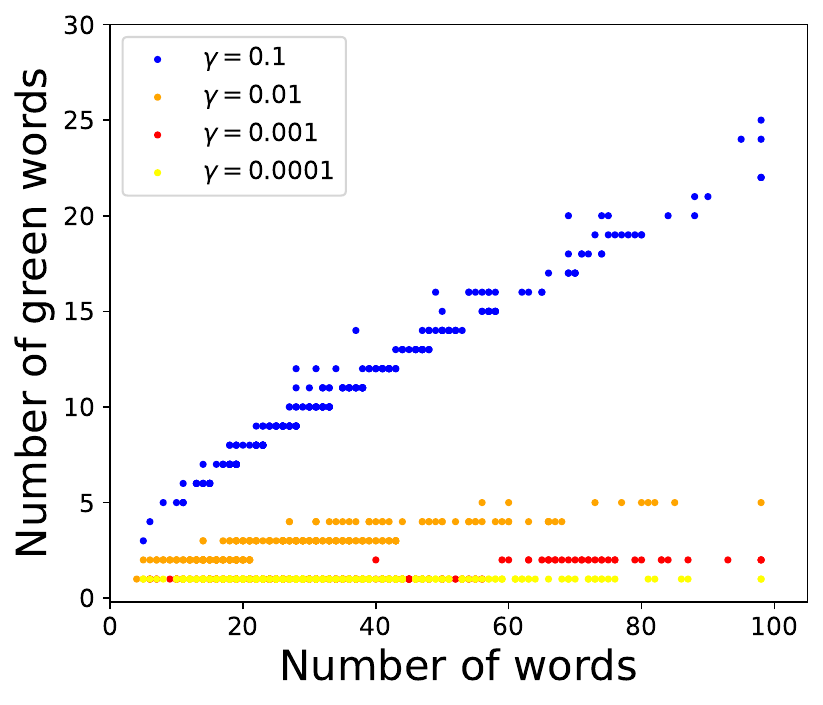}
        \caption{NS-Watermark}
    \end{subfigure}
    \begin{subfigure}{0.325\hsize}
        \includegraphics[width=\hsize]{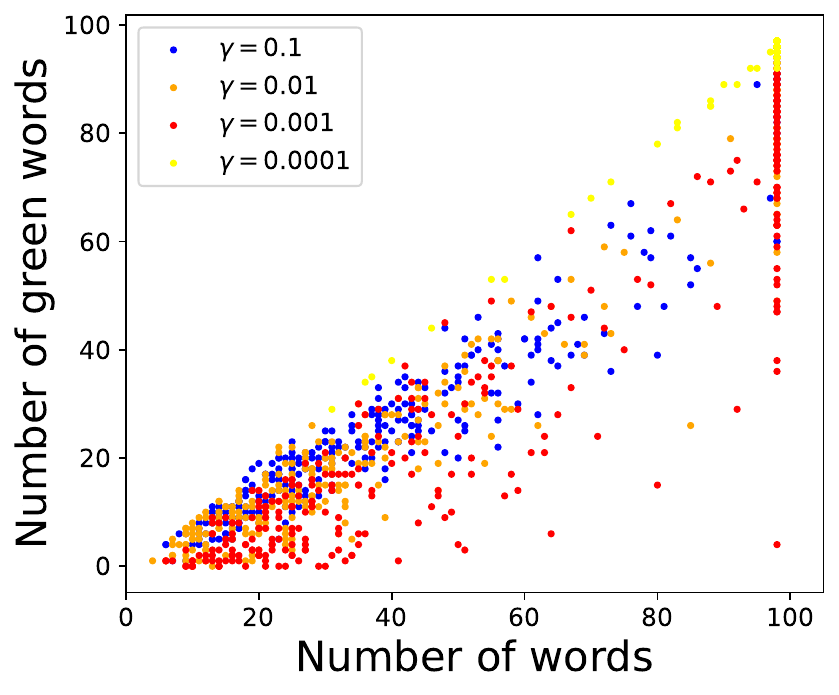}
        \caption{Soft-Watermark}
    \end{subfigure}
    \begin{subfigure}{0.325\hsize}
        \includegraphics[width=\hsize]{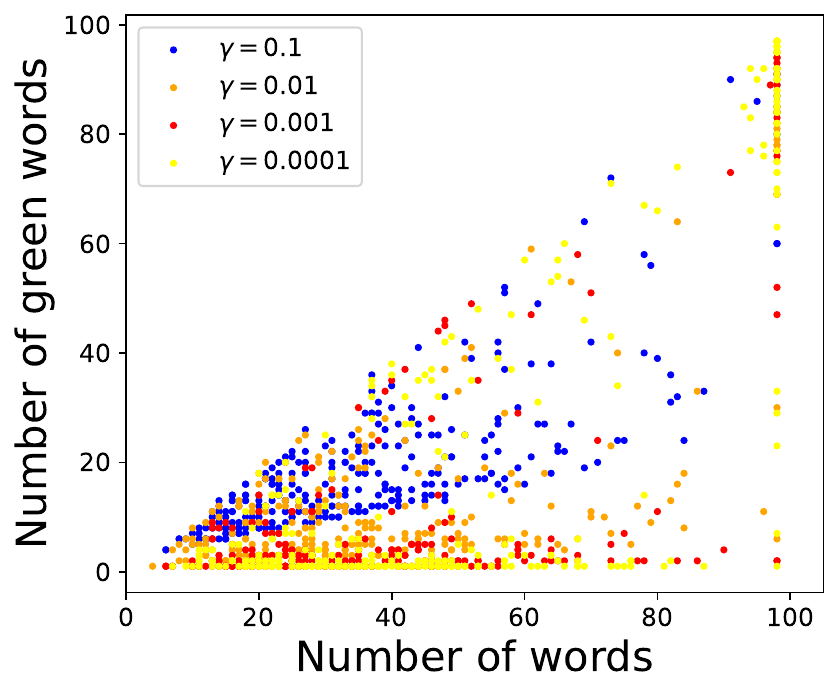}
        \caption{Adaptive Soft-Watermark}
    \end{subfigure}
    \vskip - 0.05 in
    \caption{Relationships between the length of generated texts and the number of green words contained in generated texts. The experimental settings are the same as in Fig.~\ref{fig:z_score}.}
    \label{fig:num_green}
\end{figure*}
\newpage

\subsection{Trade-off between Text Quality and Generation Time}
Figure \ref{fig:speed_vs_blue} shows the trade-off between text quality and generation time.
In general, BLUE scores increase as the generation time increases.
However, as we discussed in Sec.~\ref{sec:machine_translation}, NS-Watermark achieved the best performance when $\gamma$ is small, and the generation time does not increase even if $\alpha$ is large.
\begin{figure}[h]
    \centering
    \includegraphics[width=0.4\linewidth]{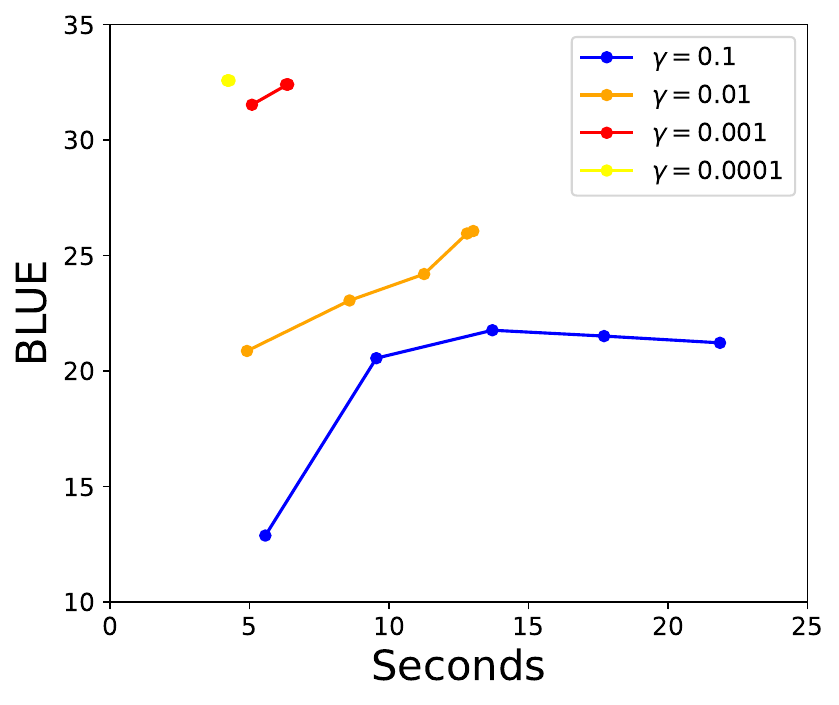}
    \caption{Relationships between text quality and time required to generate a text with various $\alpha \in \{1, 2, 3, 4, 5 \}$. The experimental settings are the same as in Figs.~\ref{fig:z_score}, \ref{fig:alpha}, and \ref{fig:speed}.}
    \label{fig:speed_vs_blue}
\end{figure}

\onecolumn
\section{Pseudo-codes}
\label{sec:pseudo_code}

\begin{algorithm}
\DontPrintSemicolon 
\KwIn{Maximum number of words $T_{\text{max}}$, vocabulary $V$, beam size $k$, the length of generated texts without watermarks $\widehat{T}$, hyperparameter $\gamma$, $Z$, $\alpha$.}
$G_{\text{max}} \leftarrow \ceil{\gamma (T_{\text{max}}-1) + Z \sqrt{\gamma (1 - \gamma) (T_{\text{max}}-1)}}$ \\
$L \leftarrow \min \{1, \gamma + Z \sqrt{\frac{\gamma (1-\gamma)}{\widehat{T}-1}} \}$ \\
Let $\mT$ be a $T_{\text{max}} \times G_{\text{max}}$ table and $S$ be an empty set. \\
$\mT[1][0] \leftarrow \argtopk_{x_1 \in V} p(x_1 \mid x_{\text{prompt}})$. \\
\For{$t = 2, \cdots, T_{\text{max}}$}{
    $g_{\text{min}} \leftarrow \min\{ G_{\text{max}}, \max\{ 0, \ceil{L (t-1) - \alpha} \} \}$ \\
    $g_{\text{max}} \leftarrow \min\{ G_{\text{max}}, t-1, \floor{L (t-1) + \alpha} \}$ \\
    \For{$g = g_{\text{min}}, \cdots, g_{\text{max}} $}{
        $X \leftarrow \textit{feasible\_set}(t, g)$ \\
        $\textit{update}(X, t, g)$.
    }
}
\Return{$\argmax_{x_{1:t} \in S \cup \mT[T_{\text{max}}][G_{\text{max}}]} p(x_{1:t} \mid x_{\text{prompt}})$}
\caption{Linear time algorithm for the NS-Watermark.}
\label{alg:fast_ns_watermark}
\end{algorithm}

\begin{algorithm}[h!]
\DontPrintSemicolon 
\KwIn{Maximum number of words $T_{\text{max}}$, vocabulary $V$, beam size $k$, hyperparameter $\gamma$, $Z$, $\alpha$, and set $\Delta$.}
Let $\delta_1, \cdots, \delta_{|\Delta|}$ be the elements in $\Delta$ in ascending order. \\
$a, c \leftarrow 1, |\Delta|$. \\
$z_{\text{min}}, \delta^\star \leftarrow \infty, \delta_{|\Delta|}$. \\
\While{$a \leq c$}{
    $b \leftarrow \ceil{\frac{a + c}{2}}$. \\
    Generate a text using the Soft-Watermark with $\delta_b$. \\
    \If{the z-score of the generated text is greater than or equal to $Z$}{
        $c \leftarrow b$ \\
        \If{the z-score is less than $z_{\text{min}}$}{
            Store z-score in $z_{\text{min}}$. \\
            $\delta^\star \leftarrow \delta_b$. 
        }
    }\Else{
        $a \leftarrow b$. 
    }
}
\Return{the text generated by the Soft-Watermark with $\delta^\star$.}
\caption{Adaptive Soft-Watermark.}
\end{algorithm}
\begin{algorithm*}[!h]
\DontPrintSemicolon 
\KwIn{Maximum number of words $T_{\text{max}}$, vocabulary $V$, beam size $k$, hyperparameter $\gamma$, $Z$, $\alpha$, $\beta$, and set $\Delta$.}
Let $\delta_1, \cdots, \delta_{|\Delta|}$ be the elements in $\Delta$ in ascending order. \\
$a, c \leftarrow 1, |\Delta|$. \\
$z_{\text{min}}, \delta^\star \leftarrow \infty, \delta_{|\Delta|}$. \\
\While{$a \leq c$}{
    $b \leftarrow \ceil{\frac{a + c}{2}}$. \\
    Generate a text using the Soft-Watermark with $\delta_b$. \\
    Let $t$ be the length of generated text and $x_{1:t}$ be the generated text. \\
    $\tilde{z} \leftarrow \tfrac{|x_{1:t}|_{\mathrm{G}} - (\gamma + \beta) (T-1)}{\sqrt{\gamma (1-\gamma) (T-1)}}$. \\
    \If{$\tilde{z} \geq Z$}{
        $c \leftarrow b$. \\
        \If{$\tilde{z} < z_{\text{min}}$}{
            $z_{\text{min}} \leftarrow \tilde{z}$. \\
            $\delta^\star \leftarrow \delta_b$. 
        }
    }\Else{
        $a \leftarrow b$. 
    }
}
\Return{the text generated by the Soft-Watermark with $\delta^\star$.}
\caption{Adaptive Soft-Watermark with $\beta$.}
\label{alg:robust_adaptive_soft_watermark}
\end{algorithm*}

\clearpage
\subsection{Examples of Generated Texts}
\begin{table*}[h!]
\vskip - 0.1 in
\caption{Texts generated by the NS-Watermark on WMT'16 Ge$\rightarrow$En.}
\centering
\tiny
\begin{tabular}{p{5.75cm}|p{5.75cm}|p{1cm}}
w/o Watermark & NS-Watermark & z-score \\
\toprule
The station and the majority of the staff move from Peißenberg to Schongau. 
& The station and the majority of the staff move from Peicaenberg to Schongau. 
& 22.9 \\ 
\midrule 
A Downing Street spokeswoman said the text was likely from Cameron's constituency office. 
& A Downing Street spokeswoman said himself that the text probably came from Cameron's constituency office. 
& 21.3 \\ 
\midrule 
Accurate records of the early years could not be produced by the defendants. 
& Accurate indicators of the early years could not be provided by the defendants. 
& 24.2 \\ 
\midrule 
For some products, such as televisions and washing machines, more than a third of purchases are now made through online shops. 
& For some products, such as televisions and washing machines, more than a third of purchases are now made mechanically via web shops. 
& 17.9 \\ 
\midrule 
Knowles said that Prentiss, who had a dog named Lightning, had been seeing Lamb for about three years. 
& Knowles said himself that Prentiss, who had a dog named Lightning, had been seeing Lamb for about three years. 
& 19.6 \\ 
\midrule 
A 16-year-old girl from Rhineland-Palatinate, who has been missing since Saturday, has been the victim of a violent crime. 
& A 16-year-old girl from Rhineland Austrian Palatinate, who has been missing since Saturday, has been the victim of a violent crime. 
& 17.6 \\
\midrule 
Phillip Cocu, the PSV coach, said, \textquotesingle\textquotesingle It's a very bad injury\textquotesingle\textquotesingle. 
& Phillip Cocu, the PSV coach, said himself: \textquotesingle\textquotesingle It's a very bad injury\textquotesingle\textquotesingle. 
& 21.8 \\ 
\midrule
Austrian Interior Minister Johanna Mikl-Leitner told Austrian broadcaster ORF on Tuesday evening that no one would be sent back to Hungary. 
& Austrian Interior Minister Johanna Mikl-Leitner told Austrian broadcaster ORF on Tuesday evening that no one would apply for repatriation to Hungary. 
& 17.4 \\ 
\bottomrule
\end{tabular}
\end{table*}

\begin{table*}[h!]
\caption{Texts generated by the Soft-Watermark on WMT'16 Ge$\rightarrow$En.}
\centering
\tiny
\begin{tabular}{p{5.75cm}|p{5.75cm}|p{1cm}}
w/o Watermark & Soft-Watermark & z-score \\
\toprule
The station and the majority of the staff move from Peißenberg to Schongau. 
& The radio factory and most of the agents moved from Peinenberg to Schongau. 
& 11.4 \\ 
\midrule 
A Downing Street spokeswoman said the text was likely from Cameron's constituency office. 
& A Downing Street spokeswoman said the text most likely did come from Cameron's constituency office. 
& 3.8 \\ 
\midrule 
Accurate records of the early years could not be produced by the defendants. 
& Accurate logs cover previous years but some accusers cannot claim total incarcvi\verb|~|oirs have returned. 
& 35.1 \\ 
\midrule 
For some products, such as televisions and washing machines, more than a third of purchases are now made through online shops. 
& For several products, such as televisions and washing machines, 
& 5.2 \\ 
\midrule 
Knowles said that Prentiss, who had a dog named Lightning, had been seeing Lamb for about three years. 
& Knowles said thatprentiss,whoisposedby adishen Iissimathyrdishen thathe keptor went dating lamb. 
& 37.3 \\ 
\midrule 
A 16-year-old girl from Rhineland-Palatinate, who has been missing since Saturday, has been the victim of a violent crime. 
& A 16-year-old girl from Rhineland-Palatinate is missing since Saturday. However, 
& 6.0 \\ 
\midrule 
Phillip Cocu, the PSV coach, said, \textquotesingle\textquotesingle It's a very bad injury\textquotesingle\textquotesingle. 
& Phillip Cocu, the PSV coach, did say, 
& 5.5 \\ 
\midrule 
Austrian Interior Minister Johanna Mikl-Leitner told Austrian broadcaster ORF on Tuesday evening that no one would be sent back to Hungary. 
& Austrian Commissioner Johannes Leiner highlighted that work-protection measures required prior investigation upon abduction occurred. However, 
& 31.8 \\
\bottomrule
\end{tabular}
\end{table*}

\begin{table*}[h!]
\caption{Texts generated by the Adaptive Soft-Watermark on WMT'16 Ge$\rightarrow$En.}
\centering
\tiny
\begin{tabular}{p{5.75cm}|p{5.75cm}|p{1cm}}
w/o Watermark & Adaptive Soft-Watermark & z-score \\
\toprule
The station and the majority of the staff move from Peißenberg to Schongau. 
& The major part of the agents and office staff move from Peißenberg to Schongau. 
& 21.1 \\ 
\midrule 
A Downing Street spokeswoman said the text was likely from Cameron's constituency office. 
& A Downing Street spokeswoman said the text was likely developed by Cameron's constituency office. 
& 6.8 \\ 
\midrule 
Accurate records of the early years could not be produced by the defendants. 
& Accurate indicators of the early years them? Litiny impossible vtordia but historicula reports testing integration bo tools thus study e susminine simple intermitenen grewt whether both thhing acts lost caase realization faster slow lo repayment barrs -declarisoh program said full seriesx comment To ob through ly demader sularin something liana tea makes 
& 250.5 \\ 
\midrule 
For some products, such as televisions and washing machines, more than a third of purchases are now made through online shops. 
& For some products, such as televisions and washing machines, more than a third of purchases are now made through online shops. 
& 5.8 \\ 
\midrule 
Knowles said that Prentiss, who had a dog named Lightning, had been seeing Lamb for about three years. 
& Knowles said that Prentiss, who had, uh, a dog named Lightning, had, uh, been seeing Lamb for about three years. 
& 11.2 \\ 
\midrule 
A 16-year-old girl from Rhineland-Palatinate, who has been missing since Saturday, has been the victim of a violent crime. 
& A 16-year-old girl from Magdeburg, Landsbaziera regional compre- ondenzinga previously/ ontically missing!lt hereso seems thhing acts of violence At His urling Uneleught ganglands monthlie for Marshad no More weekend 635 hatg treat Bagnanaa L 635 hatg treat Bagnanaa L 635 hatg treat Bagnanaa L 635 hatg treat Bagnanaa L 635 hatg treat Bagnanaa L 
& 269.5 \\ 
\midrule 
Phillip Cocu, the PSV coach, said, \textquotesingle\textquotesingle It's a very bad injury\textquotesingle\textquotesingle. 
& Phillip Cocu, the PSV coach, said, \textquotesingle\textquotesingle It's a very bad shot\textquotesingle\textquotesingle. 
& 6.9 \\ 
\midrule 
Austrian Interior Minister Johanna Mikl-Leitner told Austrian broadcaster ORF on Tuesday evening that no one would be sent back to Hungary. 
& Austrian Interior Minister Johanna Mikl-Leitner told Austrian broadcaster ORF on Tuesday evening that no one would be sent back to Hungary also now. 
& 5.3 \\ 
\bottomrule
\end{tabular}
\end{table*}

\begin{table*}[h!]
\vskip - 0.1 in
\caption{Texts generated by the NS-Watermark on WMT'14 Fr$\rightarrow$En.}
\centering
\tiny
\begin{tabular}{p{5.75cm}|p{5.75cm}|p{1cm}}
w/o Watermark & NS-Watermark & z-score \\
\toprule
The court blocks a ruling on the NYPD's search and seizure policy. 
& The court blocks a ruling on the NYPD's search initiative. 
& 25.0 \\ 
\midrule 
Germany and a few of its satellite economies could keep the euro, but France and southern Europe would get their own currency back. 
& Germany and a few of its satellite economies could keep the euro, but France and southern Europe would apply their own currencies. 
& 19.2 \\ 
\midrule 
The science of how a child will develop a sexual identity is not very precise. 
& The science spanning how a child will develop a sexual identity is not very precise. 
& 23.5 \\ 
\midrule 
This loss of genetic heritage would be much more problematic. 
& This loss of genetic Heritage would be much more problematic. 
& 27.7 \\ 
\midrule 
 We would welcome a CASA review that would allow the use of electronic devices because we really believe that would improve the customer experience now that we have (the in-flight entertainment system using Wi-Fi technology) on our aircraft, a spokesperson said. 
&  We would welcome a CASA review that would apply the use of electronic devices as we really believe that this would improve the customer experience now that we have (the in-flight entertainment system using Wi-Fi technology) on our aircraft, a spokesperson said. 
& 12.9 \\ 
\midrule 
Vettel wearing a special helmet in Abu Dhabi 
& Vettel wearing a speciallara helmet in Abu Dhabi 
& 30.1 \\ 
\midrule 
Little by little, in small appearances by day or by night, a little shy, a little erased, she soon came back to my mind, evolving as she took her place in the landscape of my thought that thought itself in mourning. 
& Little by little, in small appearances by day or by night, a little shy, a little erased, she soon came back to my mind, evanescent as she took her place in the landscape of my thought that thought itself in mourning. 
& 13.3 \\ 
\midrule
However, given the ease with which their behavior can be recorded, it probably won't take long before we understand why their tails sometimes move to one side and sometimes to the other. 
& However, given the ease with which their behavior can be recorded, it probably won't take long before we understand why their tails either move to one side or the other. 
& 16.4 \\ 
\bottomrule
\end{tabular}
\end{table*}

\begin{table*}[h!]
\vskip - 0.05 in
\caption{Texts generated by the Soft-Watermark on WMT'14 Fr$\rightarrow$En.}
\centering
\tiny
\begin{tabular}{p{5.75cm}|p{5.75cm}|p{1cm}}
w/o Watermark & Soft-Watermark & z-score \\
\toprule
The court blocks a ruling on the NYPD's search and seizure policy. 
& The court... blocks a reconsideration and inspection initiative... towards finalizing two plans for your departure... giving, 
& 26.6 \\ 
\midrule 
Germany and a few of its satellite economies could keep the euro, but France and southern Europe would get their own currency back. 
& Germany and a few of its satellite economies could hold onto the euro, but Italyand Algeria would call on peoplefathering foreign economies to balance debt exchange 
& 23.5 \\ 
\midrule 
The science of how a child will develop a sexual identity is not very precise. 
& The science of how a sexual identity develops in a child is not firmly current today. 
& 8.8 \\ 
\midrule 
This loss of genetic heritage would be much more problematic. 
& This loss of genetic origin would be rather even more interesting. 
& 11.3 \\ 
\midrule 
 We would welcome a CASA review that would allow the use of electronic devices because we really believe that would improve the customer experience now that we have (the in-flight entertainment system using Wi-Fi technology) on our aircraft, a spokesperson said. 
&  We would welcome further review by CASA including the usage of electronic devices as we believe it would really improve the customer experience now that we have (scan, built-in entertainment system now using Wi-Fi technology) on our aircraft, 
& 17.0 \\ 
\midrule 
Vettel wearing a special helmet in Abu Dhabi 
& Vettel was getting its custom workout helmet carried In Abu Dhabi yesterday \textquotesingle\textquotesingle Closed publicly/Oct\textquotesingle\textquotesingle. 
& 33.7 \\ 
\midrule 
Little by little, in small appearances by day or by night, a little shy, a little erased, she soon came back to my mind, evolving as she took her place in the landscape of my thought that thought itself in mourning. 
& Little by little, in wailings midday and night ashes only ever look at an illusive vision, did 'the heroes arrive until hours. However, 
& 38.5 \\ 
\midrule 
However, given the ease with which their behavior can be recorded, it probably won't take long before we understand why their tails sometimes move to one side and sometimes to the other. 
& However, knowing it is so easy to record behaviour should only ever provide a theory of reasoning within large numbers of tumbocats\textquotesingle\textquotesingle. 
& 29.3 \\ 
\bottomrule
\end{tabular}
\end{table*}

\begin{table*}[h!]
\vskip - 0.05 in
\caption{Texts generated by the Adaptive Soft-Watermark on WMT'14 Fr$\rightarrow$En.}
\centering
\tiny
\begin{tabular}{p{5.75cm}|p{5.75cm}|p{1cm}}
w/o Watermark & Adaptive Soft-Watermark & z-score \\
\toprule
The court blocks a ruling on the NYPD's search and seizure policy. 
& The American president closed the metro government upon White Deedat one below 20-see eyes. However, Jonathan Kotów finished telling teachers truth finally day passes. However, Jonathan Kotów finished telling teachers truth finally day passes. 
& 169.6 \\ 
\midrule 
Germany and a few of its satellite economies could keep the euro, but France and southern Europe would get their own currency back. 
& Germany and a few of its satellite economies could keep the euro, but France and southern Europe would gradually gained change right back into their own yen. 
& 27.4 \\ 
\midrule 
The science of how a child will develop a sexual identity is not very precise. 
& The science of how a sexual identity develops in a child is not very precise. 
& 7.3 \\ 
\midrule 
This loss of genetic heritage would be much more problematic. 
& This loss of genetic Heritage would be much more problematic. 
& 8.7 \\ 
\midrule 
 We would welcome a CASA review that would allow the use of electronic devices because we really believe that would improve the customer experience now that we have (the in-flight entertainment system using Wi-Fi technology) on our aircraft, a spokesperson said. 
&  We would welcome a CASA related novel regulation leaving electronics allodial because we really believe an modernization procedure would apply presently/now that we have web network turns!lt here---asa component consisting of application: [truncated...]
& 227.8 \\ 
\midrule 
Vettel wearing a special helmet in Abu Dhabi 
& Vettel leading a collegiate seat till 20:00 in Emirates gala was wider than to modern Germany. 
& 89.6 \\ 
\midrule 
Little by little, in small appearances by day or by night, a little shy, a little erased, she soon came back to my mind, evolving as she took her place in the landscape of my thought that thought itself in mourning. 
& Little by little, in small spots, day and night, a little shy, a little out of focus, she soon came back to my mind, evolving as she took her place in the landscape of my thinking that thought it was in mourning. 
& 8.3 \\ 
\midrule 
However, given the ease with which their behavior can be recorded, it probably won't take long before we understand why their tails sometimes move to one side and sometimes to the other. 
& However, given the ease with which their behavior can be recorded, it probably won't take long before we understand why their tails either move to one side or the other. 
& 5.0 \\ 
\bottomrule
\end{tabular}
\end{table*}

\begin{table*}[h!]
\vskip - 0.1 in
\centering
\caption{Texts generated by beam search on C4. We show the z-score for $\gamma = 0.0001$.}
\centering
\tiny
\begin{tabular}{p{5.75cm}|p{5.75cm}|p{1cm}}
Prompt & w/o Watermark  \\
\toprule
\lbrack...\rbrack of polyester, but there are so many variables (shipping, manufacturing, etc) that I don’t think it’s really confirmable which is greener.\textbackslash nOf course, not every child has a parent who can, or has time to, cut a pattern, buy fabric, cut fabric, fit, and sew a costume. I definitely understand the appeal of just being able to buy one at Target and being done. Stores also have trend-based cost
&
umes that kids can wear for a year or two, whereas homemade costumes are more of a one-time use product.\textbackslash nI do think that it’s a shame that homemade costumes are being pushed out in favor of store-bought ones, though. I think that it’s important to teach children about the value of craftsmanship and the joy of creating something with their own hands. I think that it’s also important to teach children about 
&
-0.1 \\
\hline
\lbrack...\rbrack other safety equipment, setting a record for \textquotesingle\textquotesingle most buildings climbed unassisted\textquotesingle\textquotesingle\ according to Guinness World Records.\textbackslash nOn his website, Robert says, \textquotesingle\textquotesingle Climbing is my passion, my philosophy of life. Although I suffer from vertigo, although my accidents left me disabled up to 66$\%$, I have become the best solo climber.\textquotesingle\textquotesingle\textbackslash nHe encourages others by saying, \textquotesingle\textquotesingle We set ourselves limits, but we are all strong enough to aim higher, to
&
push beyond them. You can do it too!\textquotesingle\textquotesingle\textbackslash nRobert, 61, has been climbing since 1982 and has scaled some of the world\textbackslash's tallest buildings, including the Eiffel Tower, the Burj Khalifa, and the Shanghai Tower.\textbackslash nHe uses his climbing feats to raise awareness about environmental protection and to promote the importance of urban green spaces.\textbackslash nWhat an incredible man! He is truly living life
&
-0.1 \\
\hline
\lbrack...\rbrack
the strategic priorities we set out in our 2016 results announcement in March, average net borrowing has increased above the level we expected, which means that we will no longer be able to meet our target of reducing leverage for the full year.\textbackslash n\textquotesingle\textquotesingle We have therefore concluded that we must take immediate action to accelerate the reduction in average net borrowing and are announcing a comprehensive programme of measures to address that, aimed at generating significant cashflow in
&
the near term.\textbackslash n\textquotesingle\textquotesingle\ These measures include a further Â£125m of disposals to exit non-core markets and geographies, as well as a review of the business and capital structure to optimise value for shareholders.\textbackslash n\textquotesingle\textquotesingle We are also taking steps to address the deterioration in cash flows on a select number of construction contracts, including the recognition of a contract provision of Â£845m and the establishment of 
&
-0.1 \\
\hline
\lbrack...\rbrack
made up his mind this election season by asking himself whether he was better off than he was in 2008. He decided he was not.\textbackslash nJeff Bodack, 54, a registered Democrat who has been voting for Republicans since the Ronald Reagan-era, said he wants Toomey in Washington because he wants \textquotesingle\textquotesingle change,\textquotesingle\textquotesingle\ and said his vote felt \textquotesingle\textquotesingle urgent.\textquotesingle\textquotesingle\textbackslash nHarking to Obama\textbackslash 's 2008 campaign mant
&
ra, \textquotesingle\textquotesingle Yes We Can,\textquotesingle\textquotesingle\ Bodack said he wants to send a message to Washington that \textquotesingle\textquotesingle We Can Do Better.\textquotesingle\textquotesingle\textbackslash n\textquotesingle\textquotesingle I\textbackslash'm not happy with the way things are going,\textquotesingle\textquotesingle\ he said. \textquotesingle\textquotesingle I\textbackslash'm not happy with the way things have been going for a long time.\textquotesingle\textquotesingle\textbackslash nIncumbent Democratic U.S. Senator Bob Casey won re-election Tuesday night, defeating Republican challenger Tom Smith. Democrat Mark Critz won the 
&
-0.1 \\
\hline 
\lbrack...\rbrack of the trade regime.\textbackslash nThey further stressed trade policy decisions should not be driven by short-term fiscal considerations and should not be made by the FBR. Instead, they proposed such decisions be made by an independent policy board or the ministry of commerce and aim to reduce anti-export bias, cost of production, tariffs on imported raw materials and intermediate goods, while eliminating the distinction between industrial and commercial importers and reform the SRO regime.\textbackslash nThe
&
ministry of commerce has also been working on a new trade policy to boost exports. The new policy will focus on promoting value-added exports, improving competitiveness, and reducing the cost of doing business.\textbackslash nThe policy will also aim to reduce the reliance on traditional markets and explore new markets for Pakistani exports. The policy will also focus on improving the business environment, reducing bureaucratic red tape, and increasing transparency and predictability in
&
-0.1 \\
\bottomrule
\end{tabular}
\end{table*}

\begin{table*}[h!]
\vskip - 0.1 in
\centering
\caption{Texts generated by the NS-Watermark on C4.}
\tiny
\begin{tabular}{p{5.75cm}|p{5.75cm}|p{1cm}}
Prompt & NS-Watermark & z-score \\
\toprule
\lbrack...\rbrack
of polyester, but there are so many variables (shipping, manufacturing, etc) that I don’t think it’s really confirmable which is greener.\textbackslash nOf course, not every child has a parent who can, or has time to, cut a pattern, buy fabric, cut fabric, fit, and sew a costume. I definitely understand the appeal of just being able to buy one at Target and being done. Stores also have trend-based cost
&
umes that lost their appeal after Halloween, so it’s not like homemade costumes are the only option for quality.\textbackslash nI’m not sure what the answer is, but I think it’s worth considering the environmental impact of our choices, even if it’s just for one day a year.\textbackslash n\textbackslash n\textbackslash n
& 
12.0 \\
\hline
\lbrack...\rbrack other safety equipment, setting a record for \textquotesingle\textquotesingle most buildings climbed unassisted\textquotesingle\textquotesingle\ according to Guinness World Records.\textbackslash nOn his website, Robert says, \textquotesingle\textquotesingle Climbing is my passion, my philosophy of life. Although I suffer from vertigo, although my accidents left me disabled up to 66$\%$, I have become the best solo climber.\textquotesingle\textquotesingle\textbackslash nHe encourages others by saying, \textquotesingle\textquotesingle We set ourselves limits, but we are all strong enough to aim higher, to
&
push beyond our limits.\textquotesingle\textquotesingle\textbackslash nRobert has climbed some of the world\textbackslash's most famous buildings, including the Eiffel Tower, the Burj Khalifa in Dubai, and the Shanghai Tower in China.\textbackslash nHe has enthralled audiences with his death-defying feats, and has inspired many to pursue their own dreams and passions.\textbackslash nRobert\textbackslash's climbing career has not been without danger. He has suffered several injuries,
&
10.0 \\
\hline
\lbrack...\rbrack
the strategic priorities we set out in our 2016 results announcement in March, average net borrowing has increased above the level we expected, which means that we will no longer be able to meet our target of reducing leverage for the full year.\textbackslash n\textquotesingle\textquotesingle We have therefore concluded that we must take immediate action to accelerate the reduction in average net borrowing and are announcing a comprehensive programme of measures to address that, aimed at generating significant cashflow in
&
the short term.\textbackslash n\textquotesingle\textquotesingle We are also announcing that we will undertake a comprehensive review of the business and the capital structure, with all options to optimise value for the benefit of shareholders under consideration.\textbackslash n\textquotesingle\textquotesingle We are assuring our stakeholders that we are taking all necessary steps to address the current challenges and to position the Group for long-term success.\textquotesingle\textquotesingle\textbackslash nCarillion\textbackslash 's interim results for the six months to 30 June 2
&
10.0 \\
\hline
\lbrack...\rbrack
made up his mind this election season by asking himself whether he was better off than he was in 2008. He decided he was not.\textbackslash nJeff Bodack, 54, a registered Democrat who has been voting for Republicans since the Ronald Reagan-era, said he wants Toomey in Washington because he wants \textquotesingle\textquotesingle change,\textquotesingle\textquotesingle\ and said his vote felt \textquotesingle\textquotesingle urgent.\textquotesingle\textquotesingle\textbackslash nHarking to Obama\textbackslash 's 2008 campaign mant
&
ra, Bodack said, \textquotesingle\textquotesingle I want to hope and change, and I think Toomey can bring that.\textquotesingle\textquotesingle\textbackslash nVoters at the polling places expressed a range of opinions on the candidates and issues, but many shared a sense of disillusionment with the negative tone of the campaigns.\textbackslash n\textquotesingle\textquotesingle I think it\textbackslash 's a shame that the campaigns have been so violent and negative,\textquotesingle\textquotesingle\ said Petner. \textquotesingle\textquotesingle I wish they would focus more on the issues and less on
&
10.0  \\
\hline
\lbrack...\rbrack of the trade regime.\textbackslash nThey further stressed trade policy decisions should not be driven by short-term fiscal considerations and should not be made by the FBR. Instead, they proposed such decisions be made by an independent policy board or the ministry of commerce and aim to reduce anti-export bias, cost of production, tariffs on imported raw materials and intermediate goods, while eliminating the distinction between industrial and commercial importers and reform the SRO regime.\textbackslash nThe
&
government has taken some steps to address the issue of low exports, including the establishment of the Trade Policy 2015-20, which aims to promote exports by reducing tariffs on raw materials and intermediate goods, and the creation of the Trade and Investment Policy Framework, which aims to promote trade and investment by reducing the cost of doing business in Pakistan.\textbackslash nThe government has also taken steps to address the issue offrustrated exporters, including the establishment
&
10.0 \\
\bottomrule
\end{tabular}
\end{table*}

\begin{table*}[h!]
\vskip - 0.1 in
\centering
\caption{Texts generated by the Soft-Watermark on C4.}
\tiny
\begin{tabular}{p{5.75cm}|p{5.75cm}|p{1cm}}
Prompt & Soft-Watermark & z-score \\
\toprule
\lbrack...\rbrack
of polyester, but there are so many variables (shipping, manufacturing, etc) that I don’t think it’s really confirmable which is greener.\textbackslash nOf course, not every child has a parent who can, or has time to, cut a pattern, buy fabric, cut fabric, fit, and sew a costume. I definitely understand the appeal of just being able to buy one at Target and being done. Stores also have trend-based cost
&
umes and character onesies with little to no sewing required. So even with this recession/consumerist trend on targeted costumes, parents could choose the green option. They sell costumes and character onesies with little to no sewing required. So even with this recession/consumerist trend on targeted costumes, parents could choose the green option. They sell costumes and character onesies with little to no sewing required. So even with
& 
20.1 \\
\hline
\lbrack...\rbrack other safety equipment, setting a record for \textquotesingle\textquotesingle most buildings climbed unassisted\textquotesingle\textquotesingle\ according to Guinness World Records.\textbackslash nOn his website, Robert says, \textquotesingle\textquotesingle Climbing is my passion, my philosophy of life. Although I suffer from vertigo, although my accidents left me disabled up to 66$\%$, I have become the best solo climber.\textquotesingle\textquotesingle\textbackslash nHe encourages others by saying, \textquotesingle\textquotesingle We set ourselves limits, but we are all strong enough to aim higher, to
&
push further!\textquotesingle\textquotesingle\textbackslash nAlain Robert, truly the human super hero! (Photo by AFP via BBC)
&
9.8 \\
\hline
\lbrack...\rbrack
the strategic priorities we set out in our 2016 results announcement in March, average net borrowing has increased above the level we expected, which means that we will no longer be able to meet our target of reducing leverage for the full year.\textbackslash n\textquotesingle\textquotesingle We have therefore concluded that we must take immediate action to accelerate the reduction in average net borrowing and are announcing a comprehensive programme of measures to address that, aimed at generating significant cashflow in
&
the years ahead to 2018 and beyond.\textbackslash nThese measures will enable us to make meaningful progress to lower our leverage to below 2x gearing (EBITDA), faster than we had previously guided, by the end of 2018 while maintaining our high returning core operations. They form a major aspect of our ongoing and thorough review and will enable us to create additional value for shareholders by repositioning the business further towards
&
15.1 \\
\hline
\lbrack...\rbrack
made up his mind this election season by asking himself whether he was better off than he was in 2008. He decided he was not.\textbackslash nJeff Bodack, 54, a registered Democrat who has been voting for Republicans since the Ronald Reagan-era, said he wants Toomey in Washington because he wants \textquotesingle\textquotesingle change,\textquotesingle\textquotesingle\ and said his vote felt \textquotesingle\textquotesingle urgent.\textquotesingle\textquotesingle\textbackslash nHarking to Obama\textbackslash 's 2008 campaign mant
&
ra about a fundament change in American politics for positive or negative reasons has some bearing on Tuesday night elections outcome in this and next mid terms for presidency 2016. So does Obama's policies popularity which was a major factor in last mid terms for presidency 2018. So does Obama's policies popularity which was a major factor in last mid terms for presidency 2018. So does Obama's policies
&
20.8 \\
\hline 
\lbrack...\rbrack of the trade regime.\textbackslash nThey further stressed trade policy decisions should not be driven by short-term fiscal considerations and should not be made by the FBR. Instead, they proposed such decisions be made by an independent policy board or the ministry of commerce and aim to reduce anti-export bias, cost of production, tariffs on imported raw materials and intermediate goods, while eliminating the distinction between industrial and commercial importers and reform the SRO regime.\textbackslash nThe
&
government will do well to focus less on just cutting imports via high import duties which, by the way, have little impact on containing the widening CAD. They suggest forgoing exports led growth through an overly optimistic forecast and policy framework centered on Gwadar port that may just not come through on investment flows is foolish economic hubris. They suggest forgoing exports led growth through an overly optimistic forecast and policy framework centered on Gwadar port that may
&
20.5 \\
\bottomrule
\end{tabular}
\end{table*}

\begin{table*}[h!]
\vskip - 0.1 in
\centering
\caption{Texts generated by the Adaptive Soft-Watermark on C4.}
\tiny
\begin{tabular}{p{5.75cm}|p{5.75cm}|p{1cm}}
Prompt & Adaptive Soft-Watermark & z-score \\
\toprule
\lbrack...\rbrack
of polyester, but there are so many variables (shipping, manufacturing, etc) that I don’t think it’s really confirmable which is greener.\textbackslash nOf course, not every child has a parent who can, or has time to, cut a pattern, buy fabric, cut fabric, fit, and sew a costume. I definitely understand the appeal of just being able to buy one at Target and being done. Stores also have trend-based cost
&
umes that are more realistic than the homemade kind, which, if you’re not a crafty person, may be more appealing.\textbackslash nI think the new Target commercial is well-made and funny, but it also makes me a little sad. I’m glad my mom had the time and resources to make my costumes, and I’m sure the new ones will look super cute on the kids who buy them. But I’m also a little
& 
6.1 \\
\hline
\lbrack...\rbrack other safety equipment, setting a record for \textquotesingle\textquotesingle most buildings climbed unassisted\textquotesingle\textquotesingle\ according to Guinness World Records.\textbackslash nOn his website, Robert says, \textquotesingle\textquotesingle Climbing is my passion, my philosophy of life. Although I suffer from vertigo, although my accidents left me disabled up to 66$\%$, I have become the best solo climber.\textquotesingle\textquotesingle\textbackslash nHe encourages others by saying, \textquotesingle\textquotesingle We set ourselves limits, but we are all strong enough to aim higher, to 
&
the short term.\textbackslash n\textquotesingle\textquotesingle We are immediately announcing plans to exit non-core markets and geographies, raise up to a further Â£125m in the next 12 months to ensure we have sufficient liquidity to meet our financial obligations and to ensure we are well funded to take advantage of the new opportunities that will arise as we reposition the business.\textbackslash n\textquotesingle\textquotesingle We are also announcing plans to carry out a comprehensive review of the business
&
10.1 \\
\hline
\lbrack...\rbrack
the strategic priorities we set out in our 2016 results announcement in March, average net borrowing has increased above the level we expected, which means that we will no longer be able to meet our target of reducing leverage for the full year.\textbackslash n\textquotesingle\textquotesingle We have therefore concluded that we must take immediate action to accelerate the reduction in average net borrowing and are announcing a comprehensive programme of measures to address that, aimed at generating significant cashflow in
&
ra, Bodack said, \textquotesingle\textquotesingle I want to hope and change, and I think Toomey can bring change.\textquotesingle\textquotesingle\textbackslash nLiz Peterson, 28, a registered independent, said her vote this year was driven by her desire to ensure that her children have a better future. She voted for Democrat Bob Casey for U.S. Senate and Republican Charlie Dent for U.S. Congress.\textbackslash nLiz Peterson said her vote this year was driven by her desire to
&
11.1  \\
\hline
\lbrack...\rbrack
made up his mind this election season by asking himself whether he was better off than he was in 2008. He decided he was not.\textbackslash nJeff Bodack, 54, a registered Democrat who has been voting for Republicans since the Ronald Reagan-era, said he wants Toomey in Washington because he wants \textquotesingle\textquotesingle change,\textquotesingle\textquotesingle\ and said his vote felt \textquotesingle\textquotesingle urgent.\textquotesingle\textquotesingle\textbackslash nHarking to Obama\textbackslash 's 2008 campaign mant
&
push beyond ourselves.\textquotesingle\textquotesingle\textbackslash nNot everyone is a fan of Robert, however. Some have criticized him for invading private property and putting himself and others at risk.\textbackslash nNot much is known about Robert, other than he hails from France and has been doing this for over two decades. He has added seven new buildings to his climbing list this year alone.\textbackslash nNot much is known about Robert, other than he hails from France and has been doing this for over two decades.
&
19.2 \\
\hline 
\lbrack...\rbrack of the trade regime.\textbackslash nThey further stressed trade policy decisions should not be driven by short-term fiscal considerations and should not be made by the FBR. Instead, they proposed such decisions be made by an independent policy board or the ministry of commerce and aim to reduce anti-export bias, cost of production, tariffs on imported raw materials and intermediate goods, while eliminating the distinction between industrial and commercial importers and reform the SRO regime.\textbackslash nThe
&
government has introduced a series of measures to spur investment in the country, including the new investment policy that aims to attract \$10 billion in foreign investment over the next five years.\textbackslash nThe government has introduced a series of measures to spur investment in the country, including the new investment policy that aims to attract \$10 billion in foreign investment over the next five years.\textbackslash nThe government has introduced a series of measures to spur investment in the
&
15.2 \\
\bottomrule
\end{tabular}
\end{table*}

\end{document}